\documentclass{article}

\PassOptionsToPackage{numbers, compress}{natbib}
\usepackage[main, final]{neurips_2026}

\usepackage[utf8]{inputenc} 
\usepackage[T1]{fontenc}    
\usepackage{hyperref}       
\usepackage{url}            
\usepackage{booktabs}       
\usepackage{amsfonts}       
\usepackage{nicefrac}       
\usepackage{microtype}      
\usepackage[table]{xcolor}  
\usepackage{amsthm}
\usepackage{amsmath}
\usepackage{amssymb}
\usepackage{mathtools}
\usepackage{bm}
\usepackage{bbm}
\usepackage{enumerate}
\usepackage{multicol}
\usepackage{multirow}
\usepackage[capitalize,noabbrev,nameinlink]{cleveref}
\usepackage{minitoc}
\usepackage{apptools}
\usepackage{chngcntr}
\usepackage{enumitem}
\usepackage{algorithm}
\usepackage{algpseudocode}
\usepackage{array}
\usepackage{wrapfig}
\usepackage{bbding}
\usepackage{pifont}
\usepackage{graphicx}
\usepackage{subcaption}
\usepackage{tikz}
\usetikzlibrary{shapes, positioning, arrows.meta}
\usepackage{anyfontsize}

\AtAppendix{\counterwithin{theorem}{section}}

\hypersetup{
  breaklinks   = true, 
  colorlinks   = true, 
  urlcolor     = blue, 
  linkcolor    = cyan, 
  citecolor    = teal, 
}

\theoremstyle{plain}
\newtheorem{theorem}{Theorem}

\newtheorem{lemma}{Lemma}

\theoremstyle{definition}

\newtheorem{assumption}{Assumption}
\theoremstyle{remark}

\newtheoremstyle{example_style}
  {3pt}      
  {2pt}      
  {\itshape} 
  {}         
  {\bfseries}
  {.}        
  { }        
  {}         
\theoremstyle{example_style}

\newtheorem*{example*}{Example}

\crefname{definition}{Definition}{Definitions}
\crefname{theorem}{Theorem}{Theorems}
\crefname{lemma}{Lemma}{Lemmas}
\crefname{corollary}{Corollary}{Corollaries}
\crefname{remark}{Remark}{Remarks}
\crefname{problem}{Problem}{Problems}
\crefname{fact}{Fact}{Facts}
\crefname{proposition}{Proposition}{Propositions}
\crefname{example}{Example}{Examples}
\crefname{simulation}{Simulation}{Simulations}
\crefname{listthm}{Theorem}{Theorems}
\newlist{thmlist}{enumerate}{1}
\setlist[thmlist]{label=(\alph*), ref=\thetheorem\,(\alph*)}
\crefalias{thmlisti}{listthm}
\newlist{lemlist}{enumerate}{1}
\setlist[lemlist]{label=(\alph*), ref=\thelemma\,(\alph*)}
\crefname{lemlisti}{Lemma}{Lemmas}

\crefname{section}{Section}{Sections}
\crefname{figure}{Figure}{Figures}
\crefname{table}{Table}{Tables}
\crefname{algorithm}{Algorithm}{Algorithms}
\crefname{appendix}{Appendix}{Appendices}

\definecolor{ao}{rgb}{0.0, 0.5, 0.0}
\definecolor{highlight}{RGB}{255, 243, 205}

\newcommand{\mbb}[1]{\mathbb{#1}}
\newcommand{\mcal}[1]{\mathcal{#1}}

\newcommand{\R}{\mathbb{R}}
\makeatletter

\newcommand{\argmax}{\mathop{\mathrm{argmax}}}
\newcommand{\argmin}{\mathop{\mathrm{argmin}}}

\newcommand{\indep}{\perp \!\!\! \perp}

\def\Dice{\text{Dice}}

\title{On the Relaxation of Conditional Independence Assumption for Image Segmentation}

\author{%
  Zixun Wang \\
  Department of Statistics and Data Science\\
  The Chinese University of Hong Kong\\
  \texttt{1155225012@link.cuhk.edu.hk} \\
  \And
  Ben Dai \\
  Department of Statistics and Data Science\\
  The Chinese University of Hong Kong\\
  \texttt{bendai@cuhk.edu.hk} \\
}

\begin{document}

\maketitle

\doparttoc 
\faketableofcontents 

\begin{abstract}
  In semantic segmentation, a recent line of RankSEG methods directly optimizes Dice/IoU scores at inference time, improving alignment with evaluation metrics without modifying model training. Despite its theoretical and empirical success, RankSEG relies on the restrictive Conditional Independence Assumption (CIA), which ignores crucial label correlations and therefore degrades performance in ambiguous or low-contrast scenarios. 
  However, accounting for full label dependence is computationally prohibitive, requiring $\mcal{O}(d^3)$ time.
  To address this, we replace the CIA with a Spatially Localized Dependence (SLD) structure that captures local label correlations while keeping the dependence model tractable. 
  We further overcome the remaining computational bottleneck via a Reciprocal Moment Approximation coupled with a novel fixed-point optimization strategy that eliminates exhaustive search. The proposed algorithm achieves a highly practical $\mcal{O}(d \log d)$ complexity and consistently outperforms conventional argmax and CIA-based RankSEG across diverse segmentation benchmarks. Improvements are significant in low-contrast or small-object scenarios, where label dependence offers valuable signals complementary to image information for accurate segmentation.
  The code of experiments is available at \url{https://github.com/ZixunWang/RankSEG-DEP}.
\end{abstract}

\vspace{-1em}
\section{Introduction}
Semantic segmentation is a fundamental computer vision task that partitions an image into regions corresponding to different classes. It has a wide range of applications, including medical imaging~\citep{bilic2023liver}, autonomous driving~\citep{siam2018comparative}, and image editing~\citep{ling2021editgan}.

Most existing approaches generate segmentation mask by applying thresholding or argmax operations to pixel-wise probabilities estimated by neural networks~\citep{ronneberger2015u,chen2017rethinking,zhao2017pyramid,xie2021segformer}. However, they do not directly optimize the evaluation metrics commonly used in segmentation tasks, such as the Dice score or Intersection over Union (IoU), which are set-based metrics that evaluate the overall quality of the predicted mask. 

Although several surrogate loss functions have been proposed in an attempt to optimize these target metrics, such as soft-Dice/IoU loss~\citep{rahman2016optimizing,eelbode2020optimization} and Lov{\'a}sz loss~\citep{berman2018lovasz,yu2018lovasz}, they still suffer from three main limitations: (1) their theoretical consistency~\citep{tewari2007consistency} with the target metrics is either negative (Lov{\'a}sz hinge~\citep{finocchiaro2022structured}) or unknown (soft-Dice/IoU); (2) unlike cross-entropy loss, they do not produce calibrated probabilities, which are crucial for uncertainty quantification in high-stakes settings such as medical imaging~\citep{mehrtash2020confidence}; and (3) likely due to non-convexity, they are typically combined with cross-entropy loss through a weighted sum to stabilize training, introducing additional hyperparameters to tune. 

To address these limitations, a recent series of work~\citep{dai2023rankseg,wang2025rankseg}, known as RankSEG, have developed algorithms that directly optimize the Dice/IoU scores with theoretical guarantees. RankSEG takes calibrated probabilities as input and replaces the thresholding/argmax operation in the original inference step without modifying model training, making it easy to integrate into pre-trained models as a performance-enhancing post-processing module. Despite its theoretical and empirical success, RankSEG relies on the conditional independence assumption (CIA), i.e., $Y_j \indep Y_{j^{\prime}} | \bm{X}$ for $j \neq j^{\prime}$, where $\bm{X}$ is the image and $Y_j$ and $Y_{j^{\prime}}$ denote the labels of pixels $j$ and $j^{\prime}$, respectively. This assumption ignores the correlation between labels of different pixels and thus loses valuable information; for example, if two pixels are similar in both intensity and position, they are very likely to share the same label. Such information is especially useful in challenging tasks where the image itself is too noisy to provide sufficient information for accurate segmentation. 

\begin{figure}
    \vspace{-1em}
    \centering
    \begin{subfigure}[b]{0.7\textwidth}
        \centering
        \includegraphics[width=\textwidth]{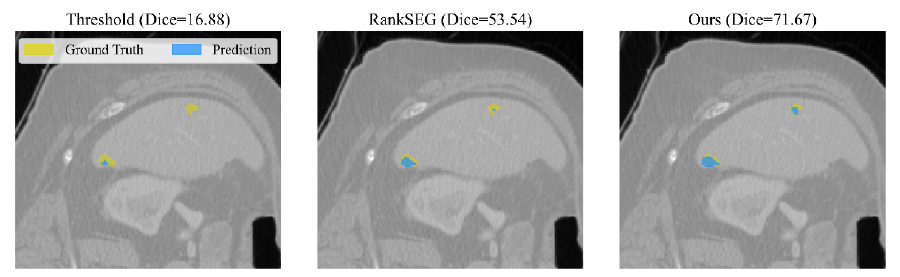}
        \caption{Comparison of different methods in tumor segmentation.}
        \label{fig:pred_compare}
    \end{subfigure}
    \hfill
    \begin{subfigure}[b]{0.28\textwidth}
        \centering
        \includegraphics[width=\textwidth]{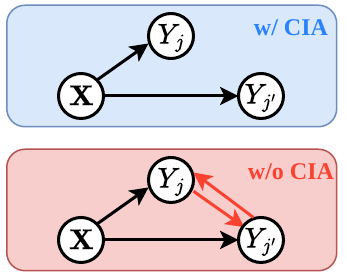}
        \caption{Graphical model.}
        \label{fig:graph}
    \end{subfigure}
    \caption{(a) The intensity contrast, and consequently the foreground probabilities, are uniformly low, so traditional thresholding barely detects any tumor and RankSEG almost entirely misses the tumor in the upper right. In contrast, our proposed method more fully recovers the tumor by leveraging label dependencies within the tumor region. (b) When predicting $Y_{j^{\prime}}$, only information from $\bm{X}$ is used under CIA, whereas both $\bm{X}$ and the neighboring label $Y_j$ are utilized when the CIA is not assumed.}
    \label{fig:compare}
    \vspace{-1.5em}
\end{figure}

\cref{fig:pred_compare} illustrates such an example: the tumor region is hardly distinguishable due to low intensity contrast in the image, and both traditional thresholding and RankSEG fail to recover the tumor in the upper right. In contrast, our proposed method, which leverages the label dependencies within this region, successfully captures the tumor more completely. Using graphical models, \cref{fig:graph} depicts the information flow of methods with and without CIA: when predicting a particular label $Y_{j^{\prime}}$, both methods rely on the image information $\bm{X}$, typically encoded as pixel-wise probabilities produced by deep neural networks. However, the graphical model without CIA explicitly accounts for the influence of other label $Y_j$ on the prediction of $Y_{j^{\prime}}$, and vice versa.



Therefore, a natural question arises: \textbf{\textit{can we develop an efficient method that directly optimizes the Dice/IoU scores yet without relying on CIA?}} This question is challenging primarily because the computational efficiency of RankSEG critically depends on the CIA; without this assumption, the associated optimization in RankSEG incurs a complexity of $\mathcal{O}(d^3)$, where $d$ is the number of pixels in the image. Such complexity is impractical for real-world applications, since $d$ typically reaches hundreds of thousands (e.g., $d \approx 10^5$ for a $384 \times 384$ image). 


To answer this question, we relax the CIA to a more realistic Spatially Localized Dependence (SLD) structure, which is strictly weaker than the CIA and well suited to image segmentation tasks. Under SLD, we develop an algorithm within the RankSEG framework that optimizes Dice/IoU in $\mathcal{O}(d \log d)$. Our main contributions are summarized as follows:
\begin{itemize}[leftmargin=1.5em, topsep=1pt, itemsep=1pt, parsep=0pt]
  \item \textbf{Relaxation of the CIA}: We relax the CIA to SLD within the RankSEG framework, preserving valuable dependency information and improving predictions in low-contrast or ambiguous regions where independence assumptions fail.
  \item \textbf{An $\mcal{O}(d \log d)$ algorithm}: Building upon this relaxed assumption, we develop a novel and efficient algorithm grounded in the RankSEG framework that directly optimizes Dice/IoU scores without relying on the CIA, making it practical for real-world applications.
  \item \textbf{Empirical validation}: We conduct extensive experiments on multiple image segmentation datasets, demonstrating the superior performance of our method compared to CIA-based methods (argmax/thresholding and RankSEG), especially in challenging scenarios with low image contrast.
\end{itemize}

\vspace{-1em}
\section{Preliminaries}
\vspace{-1em}
We start with binary segmentation for ease of clarification. Let $\bm{X} \in \R^d, \bm{Y} \in \{0,1\}^d$ represent the random variables for an image and its corresponding segmentation mask, respectively. 
The segmentation function $\bm{\delta}: \R^d \rightarrow \{0,1\}^d$ produces a predicted mask $\bm{\delta}(\bm{x})\in\{0,1\}^d$ for a test image $\bm{x}\in\R^d$. Our objective is to find the segmentation rule $\bm{\delta}$ that maximizes the expected Dice score:
\begin{equation}
    \mbb{E}_{\bm{X}, \bm{Y}}(\Dice(\bm{\delta})) = \mbb{E}_{\bm{X}, \bm{Y}} \Big( \frac{2 \bm{\delta}^{\intercal}(\bm{X}) \bm{Y}}{\|\bm{\delta}(\bm{X})\|_1 + \|\bm{Y}\|_1} \Big). \label{eq:dice}
\end{equation}

We denote $[d] = \{1,\cdots,d\}$ as the index set of pixels, and $p_j(\bm{x}) = \mbb{P}(Y_j = 1 | \bm{x})$ as the conditional probability of pixel $j$ being a foreground pixel given the image $\bm{x}$. In practice, the probability mask $\bm{p}(\bm{x}) \in [0,1]^d$ is estimated by a deep neural network, and our focus is on how to derive the final segmentation mask from $\bm{p}(\bm{x})$. For notational convenience, we henceforth suppress the conditioning on $\bm{X}=\bm{x}$ and write, e.g., $p_j$ in place of $p_j(\bm{x})$.

\subsection{Thresholding/Argmax over probability}

The traditional approach to segmentation applies the argmax operation to the probability mask for each pixel, which, in the binary case, is equivalent to thresholding at 0.5:
\begin{equation}
    \tilde{\delta}_j = \mathbbm{1}(p_j \geq 0.5) \quad \text{for all } j \in [d].
\end{equation}
However, this method has been shown to be suboptimal with respect to Dice/IoU~\citep{dai2023rankseg,wang2025rankseg}. Intuitively, this is because it makes pixel-wise decisions, whereas Dice/IoU assesses the overall overlap between the predicted and ground-truth masks. In contrast, as we show later, RankSEG accounts for this global evaluation criterion by ranking pixels according to their contributions to the expected score.

\subsection{RankSEG}
\label{sec:rankseg}
Although the original RankSEG~\citep{dai2023rankseg} is formulated under the CIA, we present the corresponding \cref{thm:rankseg} without this assumption to better illustrate the associated computational challenges. Notably, this theorem partially coincides with the Bayes rule for optimizing F-measure in multi-label classification with label dependencies \citep{dembczynski2013optimizing}.
\begin{theorem}[\citep{dembczynski2013optimizing,dai2023rankseg}]
    \label{thm:rankseg}
    The optimal segmentation rule $\bm{\delta}^*=\argmax \mbb{E}(\Dice(\bm{\delta}))$ is given by: $\delta^*_j = \mathbbm{1}(j \in \text{Top}_{\tau^*}(\bm{S}_{1:d, \tau^*}))$,
    \begin{align}
      \text{where} \quad & \tau^* = \argmax_{\tau \in [d]} \omega_{\tau}, \quad \omega_{\tau} = \textstyle{\sum_{j \in \text{Top}_{\tau}(\bm{S}_{1:d, \tau})}} S_{j,\tau} \\
      \text{and} \quad & S_{j,\tau} = p_j \mbb{E} (\tau + \Gamma_j)^{-1} = p_j \left(\sum_{k=1}^{d} \frac{\mbb{P}(\Gamma_j=k)}{\tau + k}\right). \label{eq:S_expand}
    \end{align}
    $\Gamma_j$ distributed as $\|\bm{Y}\|_1 \mid (Y_j = 1)$, i.e., $\mbb{P}(\Gamma_j = k) = \mbb{P}(\|\bm{Y}\|_1 = k | Y_j = 1)$ for all $k \in [d]$. 
    
\end{theorem}
\textbf{\textit{Intuition.}} $\omega_{\tau}$ denotes the expected Dice score when predicting the best $\tau$ pixels as foreground, and $S_{j,\tau}$ quantifies the contribution of pixel $j$ to this score. $\Gamma_j$ is interpreted as the expected volume conditioned on $Y_j=1$, and $\tau$ as the volume we consider to predict. $\Gamma_j$ and $\tau$ correspond to $\|\bm{Y}\|_1$ and $\|\bm{\delta}(\bm{X})\|_1$ in the denominator of \eqref{eq:dice}, respectively, while $p_j$ corresponds to the numerator; this explains why $S_{j,\tau}$ represents the contribution of pixel $j$ to the expected Dice score.

Under the CIA, \citet{dai2023rankseg} further showed that $S_{j,\tau}$ is monotonically increasing in $p_j$ for any fixed $\tau$, so ranking pixels by $S_{j,\tau}$ reduces to ranking by $p_j$.
This induces a \textit{global ranking} invariant to $\tau$: the optimal segmentation simply selects the $\tau^*$ pixels with the highest probabilities.
Without the CIA, however, $S_{j,\tau}$ depends on the full joint distribution of $\bm{Y}$, and may no longer be monotone in $p_j$. Consequently, the optimal segmentation instead requires explicitly ranking pixels by $S_{j,\tau}$ for each $\tau$.
\subsection{Challenges without CIA and main ideas}

\textbf{\textit{Computational challenges.}} \cref{thm:rankseg} naturally suggests a four-step procedure to obtain the optimal segmentation: (i) compute the contribution scores $\bm{S}$ according to \eqref{eq:S_expand}; (ii) for each candidate volume $\tau \in [d]$, sort the column $\bm{S}_{1:d,\tau} \in [0,1]^d$ and compute the expected Dice score $\omega_{\tau}$ by summing its top $\tau$ entries; (iii) determine the optimal volume $\tau^*$ by maximizing $\omega_{\tau}$; and (iv) predict the top $\tau^*$ pixels in $\bm{S}_{1:d,\tau^*}$ as foreground. This procedure is computationally prohibitive for high-resolution images. The inefficiency stems from three primary bottlenecks:
\begin{enumerate}[label=(\alph*), leftmargin=1.5em, topsep=1pt, itemsep=1pt, parsep=0pt]
  \item \textbf{Modeling full dependence within $\bm{Y}$.} Although label dependence offers valuable predictive signals, modeling arbitrarily complex dependencies renders the distribution of $\Gamma_j$ intractable.
  \item \textbf{Computing $\mbb{E}(\tau + \Gamma_j)^{-1}$.} Even given the distribution of $\Gamma_j$, performing the summation for all $j, \tau \in [d]$ according to \eqref{eq:S_expand} incurs a prohibitive $\mcal{O}(d^3)$ complexity.
  \item \textbf{Re-ranking across varying $\tau$.} Evaluating $\omega_{\tau}$ for each candidate $\tau$ requires a fresh sorting over $\bm{S_{:,\tau}}$, which precludes reusing intermediate results from $\omega_{\tau}$ when computing $\omega_{\tau+1}$.
\end{enumerate}

\begin{wrapfigure}{l}{0.55\textwidth}
    \vspace{-1em}
    \begin{center}
        \includegraphics[width=0.53\textwidth]{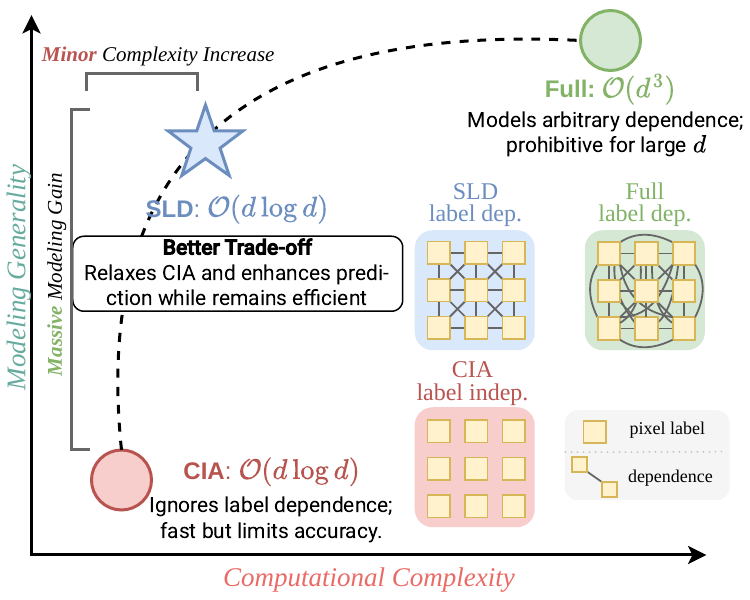}
    \end{center}
    \vspace{-0.5em}
    \caption{\textbf{Trade-off in dependence modeling.} The two endpoints represent the extremes of CIA and Full dependence. By relaxing CIA to local dependence, SLD strikes a better trade-off. The graphs depict the label dependence (dep.) under each model: none under CIA, arbitrary under Full, and neighboring-only under SLD.}
    \label{fig:tradeoff}
    \vspace{-1.7em}
\end{wrapfigure}

\textbf{\textit{Dependence Modeling.}} To overcome these challenges, we begin by reconsidering how to model the joint distribution of $\bm{Y}$. Instead of attempting to model full correlations among all $d$ labels, we propose to model \textbf{Spatially Localized Dependence (SLD)}. The key intuition is that, in image segmentation, spatially neighboring pixels exhibit stronger label correlations, whereas distant pixels behave almost independently. By restricting the dependence to these local neighborhoods, it transforms an intractable global distribution into a structured and tractable form. Crucially, SLD provides a better trade-off between efficiency and generality, as illustrated in \cref{fig:tradeoff}.



\textbf{\textit{Algorithmic Contributions.}} Building on SLD, we introduce two key algorithmic contributions to address the remaining computational bottlenecks. First, we adopt the \textbf{\textit{Reciprocal Moment Approximation}} \citep{wang2025rankseg} that replaces $\mbb{E}(\tau + \Gamma_j)^{-1}$ with $(\tau + \mbb{E}\Gamma_j)^{-1}$, decoupling the expectation from $\tau$ and reducing the problem to merely computing $\mbb{E}\Gamma_j$. SLD then enables to formulate this as a convolution, paired with an $\mcal{O}(d \log d)$ solution via the Fast Fourier Transform. Second, we propose a \textbf{\textit{fixed-point optimization}} strategy to avoid the costly re-ranking for each $\tau$. Instead of sorting from scratch each time, we maintain a single global ranking and alternate between updating $\tau$ based on the current ranking and refining the ranking from the updated $\tau$. This procedure converges quickly to an approximate $\tau^*$ without exhaustive enumeration.





\vspace{-0.5em}
\section{Method}
\label{sec:method}
\vspace{-0.5em}

In this section, we develop an efficient inference algorithm that incorporates SLD into the RankSEG framework. We first introduce the SLD assumption in \cref{sec:sld}, which relaxes the CIA to a local covariance structure tailored to image segmentation. 
These changes pose two challenges: evaluating the conditional reciprocal moment and avoiding repeated sorting across volumes.
We address the first in \cref{sec:rma} using a Reciprocal Moment Approximation, which reduces score computation to the conditional mean $\mbb{E}\Gamma_j$ and exploits the SLD structure to compute all such means efficiently by convolution and FFT. We address the second in \cref{sec:fixed-point} through a fixed-point optimization that iteratively refines the ranking and prediction volume without exhaustive search.

\subsection{Spatially Localized Dependence}
\label{sec:sld}

Completely removing the CIA introduces significant theoretical and methodological challenges due to the vast dimensionality of the joint label distribution. However, in the context of image segmentation, assuming arbitrary dependencies between all pairs of pixels is unnecessarily complex. Visual structures naturally exhibit spatial locality: \textbf{physically proximate pixels are highly likely to share the same semantic label, while the correlation between distant pixels drops off rapidly}. 

We formalize this intuition through the SLD assumption in \cref{asm:local}.

\begin{assumption}[Spatially Localized Dependence; SLD]
  \label{asm:local}
  Denote the covariance matrix of $\bm{Y}$ as $\bm{\Sigma}$, i.e., $\Sigma_{ij} = \text{Cov}(Y_i, Y_j)$. The covariance between any two labels satisfies:
  \begin{align}
    \Sigma_{ij} = \sqrt{\Sigma_{ii} \Sigma_{jj}} \cdot \mcal{K}(r(i,j)) \quad \text{with} \quad \mcal{K}(r) = \exp(-r^2/2 \theta^2), \label{eq:sld}
  \end{align}
  where $\mcal{K}$ is a Gaussian kernel, $\theta$ controls the decay rate, $r(i, j)$ is the distance between $i$ and $j$, and $\Sigma_{ii} = p_i(1-p_i)$ is the variance of $Y_i$.
\end{assumption}

The SLD explicitly encodes a structural prior: the covariance between two labels is the geometric mean of their individual variances, weighted by an exponentially decaying kernel of their spatial distance. Notably, this formulation is \textbf{strictly weaker than} the CIA, as it permits local dependence and therefore enables the exploitation of such useful information for improved performance.



Leveraging local dependence is a well-established principle in image segmentation literature, as exemplified by Conditional Random Fields (CRFs)~\citep{krahenbuhl2011efficient,farabet2012learning,chen2014semantic}. However, these methods target Maximum A Posteriori (MAP) inference, seeking the single most probable joint label map, while do not directly optimize set-based metrics like the Dice score. In contrast, we incorporate spatial priors into RankSEG to directly optimize such metrics. See \cref{sec:crf} for a more detailed discussion.



\subsection{Reciprocal Moment Approximation under SLD}
\label{sec:rma}

The primary bottleneck in evaluating the pixel contribution score $S_{j, \tau}$ lies in the reciprocal moment $\mbb{E}(\tau + \Gamma_j)^{-1}$. Due to its nonlinearity in $\Gamma_j$, exact evaluation takes $\mcal{O}(d)$ terms according to \eqref{eq:S_expand}. In contrast, evaluating the first moment, $\mbb{E} \Gamma_j$, is straightforward, as it reduces to summing the individual means. 
To bridge this gap, we adopt the Reciprocal Moment Approximation (RMA) used in \citet{wang2025rankseg} to replace $\mbb{E}(\tau + \Gamma_j)^{-1}$ with $(\tau + \mbb{E}\Gamma_j)^{-1}$, which yields the following approximation:
\begin{align}
  M_{j,\tau} = \frac{p_j}{\tau + \mbb{E} \Gamma_j} \approx p_j \mbb{E}(\tau + \Gamma_j)^{-1} = S_{j,\tau}. \label{eq:S_approx}
\end{align}
This approximation effectively interchanges the reciprocal and expectation, and can also be viewed as a first-order Taylor approximation to the reciprocal moment. Crucially, it depends only on the first moment $\mu_j := \mbb{E} \Gamma_j$. The following computation indicates that the SLD enables a more accurate estimate of $\mu_j$ by incorporating label correlations through a correction term beyond the CIA:
\begin{equation}
  \mu_j = \mbb{E} \Gamma_j = \sum_{i=1}^d \mathbb{E}\Big( Y_i \mid Y_j = 1 \Big)
  =
  \left\{
  \begin{aligned}
    & \sum_{i=1}^d \mathbb{E}( Y_i ) = \sum_{i=1}^d p_i, \quad & (\text{CIA}) \\
    & \sum_{i=1}^d p_i + \frac{\sqrt{\Sigma_{jj}}}{p_j} \sum_{i=1}^d \sqrt{\Sigma_{ii}} \mcal{K}(r(i,j)).  & (\text{SLD})
  \end{aligned}
  \right.
\end{equation}

Under the CIA, label independence reduces $\mbb{E}(Y_i | Y_j=1)$ to the marginal $\mbb{E}(Y_i)$, so $\mu_j$ collapses to a sum of probabilities that is independent of $j$, discarding critical structural information. The SLD instead leverages the covariance structure, adding a correction term that captures local dependencies. While the SLD expression for $\mu_j$ appears complex, computing all $j \in [d]$ simultaneously reveals a convolutional form. Let $q = \sum_{j=1}^d p_j$ and $\bm{\nu} = \sqrt{\text{diag}(\bm{\Sigma})} \in \mbb{R}^d$, then
\begin{align}
  & \bm{\mu} = q \cdot \bm{1} + \frac{\bm{\nu}}{\bm{p}} \cdot (\bm{\nu} \star \bm{\kappa}), \quad \text{where} \quad (\bm{\nu} \star \bm{\kappa})_j = \sum_{i=1}^{d} \bm{\nu}_i \bm{\kappa}_{d-j+i}=\sum_{i=1}^d \sqrt{\Sigma_{ii}} \mcal{K}(r(i,j)), \\
  & \quad \text{and} \quad \bm{\kappa} = (\mcal{K}(d-1), \mcal{K}(d-2), \cdots, \mcal{K}(0), \mcal{K}(1), \cdots, \mcal{K}(d-1))^{\intercal} \in \mathbb{R}^{2d-1}.
\end{align}
Here we adopt a 1D notation for simplicity; the rigorous 2D form on the image grid is given in \cref{sec:conv2d}. The convolution $(\bm{\nu} \star \bm{\kappa})$ can be efficiently evaluated via FFT in $\mathcal{O}(d \log d)$ time, with the remaining element-wise operations costing $\mathcal{O}(d)$. Moreover, this is a one-time cost: once $\bm{\mu}$ is precomputed, each RMA-based score $M_{j,\tau} = p_j / (\tau + \mu_j)$ is obtained in $\mcal{O}(1)$ cost.

With the RMA reducing the reciprocal moment to the first moment $\bm{\mu}$, which admits efficient computation as shown above, the optimal volume follows as
\begin{align}
  \tau^* = \argmax_{\tau \in [d]} \textstyle{\sum_{j \in \text{Top}_{\tau}(\bm{M}_{1:d, \tau})}} M_{j,\tau}. \label{eq:tau_opt}
\end{align}
However, unlike under CIA, the ranking of $\bm{M}_{1:d,\tau}$ depends on $\tau$, which we resolve via a fixed-point iteration in the next section.

\vspace{-0.5em}
\subsection{Fixed-point Optimization for Solving $\tau^*$}
\label{sec:fixed-point}

While RMA accelerates the evaluation of individual pixel scores, finding the optimal volume $\tau^*$ remains another bottleneck. Solving \eqref{eq:tau_opt} naively requires re-ranking the score vector $\bm{M}_{1:d,\tau}$ for each candidate $\tau$, resulting in $\mcal{O}(d^2 \log d)$ complexity. In contrast, under the CIA-based RankSEG framework~\citep{dai2023rankseg,wang2025rankseg}, the ranking is invariant across $\tau$ and is solely decided by the pixel-wise probabilities $\bm{p}$. Motivated by this observation, we isolate the global ranking from the volume selection by formulating the search for $\tau^*$ as a fixed-point iteration. 

We introduce a coupled operator $T: [d] \rightarrow [d]$ that alternates between two steps: (i) computing the global ranking based on a fixed volume, and (ii) selecting the optimal volume based on that ranking:
\begin{align}
  T(\tau) = \argmax_{\tau^{\prime} \in [d]} \pi_{\tau^{\prime}}\big( \bm{o}(\tau) \big) = \argmax_{\tau^{\prime} \in [d]} \sum_{j=1}^{\tau^{\prime}} M_{o_j(\tau), \tau^{\prime}} \quad \text{with} \quad \bm{o}(\tau) = \text{argsort}(\bm{M}_{1:d, \tau}). \label{eq:T}
\end{align}
The following lemma confirms that every optimal volume $\tau^*$ of \cref{eq:tau_opt} is a fixed point of $T$.
\begin{lemma} \label{lem:fixed}
  Denote $\text{Fix}(T) = \{\tau \in [d]: T(\tau) = \tau\}$ as the set of fixed points of $T$, then $\tau^* \in \text{Fix}(T)$.
\end{lemma}

This leads to a natural fixed-point iteration for finding $\tau^*$,
$$
\tau^{(t+1)} = T(\tau^{(t)}) = \argmax_{\tau^{\prime} \in [d]} \pi_{\tau^{\prime}}\big( \bm{o}(\tau^{(t)}) \big),
$$
where $\bm{o}( \tau^{(0)} ) := \text{argsort}(\bm{p})$ is the initial global ranking induced by the pixel-wise probabilities. The procedure terminates when the ranking stabilizes.

\begin{lemma}[Monotone Ascent] \label{lem:ascent}
  Let $\{\tau^{(t)}\}_{t \geq 0}$ be the sequence generated by 
  $\tau^{(t+1)} = T(\tau^{(t)})$, and define 
  $\Phi(\tau) = \pi_{\tau}(\bm{o}(\tau))$. Then:
  \begin{lemlist}
    \item \label{lem:ascent:i} $\Phi(\tau^{(t+1)}) \geq \Phi(\tau^{(t)})$ for all $t \geq 0$.
    \item \label{lem:ascent:ii} If the maximizer in $T$ is unique for every 
          $\tau \in [d]$, the iterates converge to a 
          fixed point of $T$.
  \end{lemlist}
\end{lemma}


The uniqueness condition in \cref{lem:ascent:ii} is exceedingly mild: non-uniqueness requires $\tau^{\prime} \mapsto \pi_{\tau^{\prime}}(\bm{o}(\tau))$ to attain identical values at two or more distinct integers, a degeneracy that defines an algebraic variety of measure zero in $(\bm{p}, \bm{\mu})$ space. In particular, when the probabilities $\bm{p}$ are produced by a neural network with continuous-valued outputs, the condition holds almost surely.



\begin{lemma}[One-Step Convergence] \label{lem:global_ranking}
  If the label distribution satisfies the following condition:
  \begin{align*}
    p_j \ge p_{j^{\prime}} \implies \mu_j \le \mu_{j^{\prime}} \quad \text{for all } j \neq j^{\prime},
  \end{align*}
  then $\bm{M}_{1:d, \tau}$ shares the same ranking as $\bm{p}$ for all $\tau$, and the iteration converges to $\tau^*$ in one step.
\end{lemma}


Surprisingly, the iteration typically converges in a single step in practice. \cref{lem:global_ranking} offers a partial explanation by establishing a sufficient condition under which the initial ranking induced by $\bm{p}$ already matches with the true global ranking. More broadly, we conjecture that the ranking induced by $\bm{p}$ remains close to that of $\bm{M}_{1:d,\,\tau^*}$ even when the condition in \cref{lem:global_ranking} does not hold exactly. The empirical prevalence of single-step convergence suggests that typical label distributions in segmentation yield a well-behaved landscape for $T$.

With the approximate cumulative sums pre-computed as described in \cref{sec:cumsum}, each iteration of \eqref{eq:T} costs $\mcal{O}(d\log d)$, dominated by the sorting operation. Given the one-step convergence observed in practice, we conclude that the overall complexity remains $\mcal{O}(d\log d)$.

\section{Error Analysis of RMA}

\cref{thm:rma} provides bounds on the reciprocal moment in terms of the first and second moments.

\begin{theorem}[Reciprocal Moment Approximation~\citep{wooff1985bounds}] \label{thm:rma}
  Let $\Gamma$ be a sum of dependent Bernoulli random variables with mean $\mu$ and variance $\sigma^2$. For any $\tau > 0$, it satisfies:
  \begin{align}
    (\mu + \tau)^{-1} \le \mbb{E}(\tau + \Gamma)^{-1} \le \frac{\mu + \sigma^2/\tau}{\mu(\mu+\tau) + \sigma^2}.
  \end{align}
  The approximation error, defined as the difference between the upper and lower bounds, satisfies:
  \begin{align}
    \mcal{E} \le \frac{\sigma^2}{\tau[\mu(\mu+\tau) + \sigma^2]}.
  \end{align}
\end{theorem}

This theorem generalizes the result of \citet{wang2025rankseg}, which applies only to sums of \textit{independent} Bernoulli random variables, to the setting where CIA does not hold. Moreover, under SLD, the covariance between labels decays exponentially, so the variance of the sum does not fluctuate significantly, growing only at the rate $\mcal{O}(d)$. Consequently, as shown in \cref{thm:error}, the error decays at the rate $\mcal{O}(d^{-1})$, coinciding with that of \citet{wang2025rankseg}, which is desirable for image segmentation tasks where $d$ is typically large. The additional regularity assumptions are mild and naturally satisfied in practice: the condition $\mu_j /d = \Theta(1)$ simply requires that the target object occupies a non-vanishing fraction of the image; while the strict positivity of the softmax function followed from the neural network inherently ensures $p_j$ is bounded away from zero.

\begin{theorem}[RMA error under SLD] \label{thm:error}
  For $\Gamma_j = \sum_{i=1}^d Y_i \mid Y_j = 1$, denote $\sigma_j^2 = \text{Var}(\Gamma_j)$. Then: 
  \begin{thmlist}
    \item Under \cref{asm:local} and $p_j \ge c > 0$ for some constant, $\sigma_j^2 = \mcal{O}(d)$ for all $j \in [d]$;
    \item Further assuming $\mu_j/d = \Theta(1)$, then $| M_{j,\tau} - S_{j,\tau} | = \mcal{O}(d^{-1})$ for all $j,\tau \in [d]$.
  \end{thmlist}
\end{theorem}

In words, the SLD structure prevents accumulated variance from growing faster than $\mcal{O}(d)$. Consequently, when the target occupies a non-vanishing fraction of the image and pixel probabilities are nondegenerate, the RMA score approximation becomes increasingly accurate as $d$ grows.

\vspace{-0.5em}
\section{Experiments}
\vspace{-0.3em}

\subsection{Settings}

\textbf{Datasets.} We evaluate our method on five publicly available datasets spanning diverse applications: LiTS~\citep{bilic2023liver} and KiTS~\citep{heller2021state} for medical imaging, DeepGlobe Land~\citep{demir2018deepglobe} for remote sensing, and ADE20K~\citep{zhou2017scene} and Cityscapes~\citep{cordts2016cityscapes} for natural image segmentation.

\textbf{Networks.} Since our method is model-agnostic and applies as a post-processing step on the predicted probability map, we evaluate it across diverse segmentation networks, including UNet~\citep{ronneberger2015u}, DeepLabV3+~\citep{chen2018encoder}, PSPNet~\citep{zhao2017pyramid}, UPerNet~\citep{xiao2018unified}, and SegFormer~\citep{xie2021segformer}. The first four are CNN-based models, while SegFormer is a transformer-based model. See \cref{sec:training} for more training details.

\textbf{Compared Methods.} To demonstrate the benefits of relaxing CIA to SLD and to evaluate our method as a training-free post-processing module, we compare against approaches using the same predicted probability map without modifying training: (i) \textit{Argmax-prob}, which assigns each pixel the highest-probability category. This is the most widely used inference rule prior to the RankSEG; (ii) \textit{CIA-RankSEG}~\citep{wang2025rankseg}, which optimizes Dice/IoU under the CIA assumption. We also evaluated \textit{CRF}~\citep{krahenbuhl2011efficient}, but since it consistently underperformed Argmax-prob, we defer its results to \cref{sec:crf}.

\textbf{Evaluation Metrics.} We report the image-wise mDice and mIoU for all datasets, following the computation of mDice$^\text{I}$ and mIoU$^\text{I}$ defined in \citet{wang2023revisiting}.

\textbf{Algorithmic Details.} We uniformly set $\theta=300$ for the Gaussian kernel in SLD across all datasets and networks. For multi-class segmentation, we adopt the same strategy as \citet{wang2025rankseg}: first apply binary RankSEG to each class independently, and then resolve conflicts by selecting the class with the highest score increment; See details in \cref{sec:multi-class}.

\subsection{Main Results}
We evaluate our method against the traditional Argmax-prob baseline and state-of-the-art CIA-RankSEG. As summarized in \cref{tab:main_results_1,tab:main_results_2}, \textbf{our method yields consistent improvements by relaxing the CIA to SLD and capturing label dependence}. We highlight the following observations:

\textbf{Model- and Dataset-Agnostic Improvements.} Gains are robust across both architectures and data domains. They hold uniformly across CNN- and transformer-based backbones, showing insensitivity to the model design, and transfer seamlessly from medical to natural segmentation. This dual robustness confirms that our method serves as a versatile post-processing module.

\textbf{Pronounced Gains on Low-Contrast Datasets.} The improvements are most pronounced on low-contrast medical images with ambiguous boundaries (LiTS, KiTS), where label dependence provides a crucial signal for accurate segmentation. For example, on LiTS and KiTS with DeepLabV3+, our method gains $+2.49$ and $+2.91$ Dice over Argmax-prob, which are $+0.37$ and $+0.51$ over CIA-RankSEG, respectively. This corroborates our motivation that capturing label dependence is especially valuable when predictions are inherently uncertain.

\begin{table}[t]
\vspace{-1em}
  \centering
  \caption{Comparison of performance (Dice and IoU \%) in LiTS and KiTS.}
  \label{tab:main_results_1}
  \small                                    
  \renewcommand{\arraystretch}{0.9}         
    \setlength\extrarowheight{0pt}          
    \begin{tabular}{c|c|c|c|c|c}
        \toprule
        \multirow{2}{*}{Model} & \multirow{2}{*}{Prediction} & \multicolumn{2}{c|}{KiTS} & \multicolumn{2}{c}{LiTS} \\
        \cline{3-4} \cline{5-6}
        & & IoU & Dice & IoU & Dice \\
        \midrule
        \multirow{3}{*}{UNet} & Argmax-prob & 51.00 & 57.36 & 38.45 & 47.58 \\
        & CIA-RankSEG & 53.54 & 60.07 & 40.70 & 50.07 \\
        & \cellcolor{highlight}\textbf{Ours} & \cellcolor{highlight}\textbf{53.92} & \cellcolor{highlight}\textbf{60.48} & \cellcolor{highlight}\textbf{40.77} & \cellcolor{highlight}\textbf{50.16} \\
        \hline
        \multirow{3}{*}{DeepLabV3+} & Argmax-prob & 54.19 & 61.16 & 38.34 & 47.38 \\
        & CIA-RankSEG & 56.22 & 63.56 & 40.09 & 49.50 \\
        & \cellcolor{highlight}\textbf{Ours} & \cellcolor{highlight}\textbf{56.70} & \cellcolor{highlight}\textbf{64.07} & \cellcolor{highlight}\textbf{40.49} & \cellcolor{highlight}\textbf{49.88} \\
        \bottomrule
    \end{tabular}
\vspace{-0.5em}
\end{table}
\begin{table}[t]
\vspace{-0.8em}
  \centering
  \caption{Comparison of performance (mDice and mIoU \%) in ADE20K, Cityscapes, and DeepGlobe.}
  \label{tab:main_results_2}
  \small
  \renewcommand{\arraystretch}{0.9}
    \setlength\extrarowheight{0pt}
    \begin{tabular}{c|c|c|c|c|c|c|c}
        \toprule
        \multirow{2}{*}{Model} & \multirow{2}{*}{Prediction} & \multicolumn{2}{c|}{ADE20K} & \multicolumn{2}{c|}{Cityscapes} & \multicolumn{2}{c}{DeepGlobe} \\
        \cline{3-8}
        & & mIoU & mDice & mIoU & mDice & mIoU & mDice \\
        \midrule
        \multirow{3}{*}{PSPNet} & Argmax-prob & 51.32 & 58.66 & 73.07 & 80.45 & 57.24 & 66.23 \\
        & CIA-RankSEG & 51.57 & 59.17 & 73.72 & 81.14 & 58.20 & 67.82 \\
        & \cellcolor{highlight}\textbf{Ours} & \cellcolor{highlight}\textbf{51.92} & \cellcolor{highlight}\textbf{59.49} & \cellcolor{highlight}\textbf{73.84} & \cellcolor{highlight}\textbf{81.27} & \cellcolor{highlight}\textbf{58.47} & \cellcolor{highlight}\textbf{68.14} \\
        \hline
        \multirow{3}{*}{DeepLabV3+} & Argmax-prob & 52.53 & 59.57 & 73.37 & 80.59 & 57.75 & 66.73 \\
        & CIA-RankSEG & 52.64 & 59.95 & 73.92 & 81.24 & 58.38 & 67.94 \\
        & \cellcolor{highlight}\textbf{Ours} & \cellcolor{highlight}\textbf{53.02} & \cellcolor{highlight}\textbf{60.27} & \cellcolor{highlight}\textbf{74.08} & \cellcolor{highlight}\textbf{81.42} & \cellcolor{highlight}\textbf{58.69} & \cellcolor{highlight}\textbf{68.27} \\
        \hline
        \multirow{3}{*}{SegFormer} & Argmax-prob & 54.09 & 61.03 & 73.32 & 80.53 & 59.87 & 68.74 \\
        & CIA-RankSEG & 54.72 & 61.92 & 74.10 & 81.38 & 60.87 & 70.14 \\
        & \cellcolor{highlight}\textbf{Ours} & \cellcolor{highlight}\textbf{55.07} & \cellcolor{highlight}\textbf{62.23} & \cellcolor{highlight}\textbf{74.23} & \cellcolor{highlight}\textbf{81.51} & \cellcolor{highlight}\textbf{61.12} & \cellcolor{highlight}\textbf{70.42} \\
        \hline
        \multirow{3}{*}{UPerNet} & Argmax-prob & 56.94 & 63.98 & 75.66 & 82.61 & 59.49 & 68.60 \\
        & CIA-RankSEG & 57.67 & 64.92 & 76.17 & 83.21 & 60.19 & 69.57 \\
        & \cellcolor{highlight}\textbf{Ours} & \cellcolor{highlight}\textbf{57.94} & \cellcolor{highlight}\textbf{65.19} & \cellcolor{highlight}\textbf{76.25} & \cellcolor{highlight}\textbf{83.31} & \cellcolor{highlight}\textbf{60.53} & \cellcolor{highlight}\textbf{69.91} \\
        \bottomrule
    \end{tabular}
  \vspace{-1em}
\end{table}

\vspace{-0.5em}
\subsection{Fined-grained Improvements}

To investigate the source of our improvements, we conduct a fine-grained analysis along two dimensions: (i) \emph{easy vs. hard cases}, where difficulty is measured by Argmax-prob performance; and (ii) \emph{small vs. large categories}, where difficulty is measured by the average pixel area of each category. In both analyses, we report relative improvements over baselines to account for varying difficulty levels.

\begin{wraptable}{r}{0.45\textwidth}
\vspace{-1em}
\centering
\caption{Relative Dice improvements over baselines on KiTS subsets, formed by the lower quantiles of Argmax-prob performance.}
\label{tab:easy_hard}
\small
\renewcommand{\arraystretch}{0.9}
\setlength\extrarowheight{0pt}
\resizebox{\linewidth}{!}{%
\begin{tabular}{c|c|c|c}
\toprule
Quantile & 0.2 & 0.4 & 0.6 \\
\midrule
Argmax-prob & +14.37 \% & +7.46 \% & +3.97 \% \\
CIA-RankSEG & +2.51 \% & +1.04 \% & +0.48 \% \\
\bottomrule
\end{tabular}
}
\vspace{-1em}
\end{wraptable}

\textit{Easy vs. Hard cases:} We split KiTS into subsets by the lower quantile of Argmax-prob performance, where lower values correspond to harder cases. As shown in \cref{tab:easy_hard}, our method yields larger gains on harder subsets, while improvements on easier subsets become modest. This confirms that modeling label dependence is most valuable when segmentation is challenging.

\begin{wraptable}{r}{0.45\textwidth}
\vspace{-1em}
\centering
\caption{Relative Dice improvements over baselines on ADE20K across category groups, defined by average pixel area.}
\label{tab:small_large}
\small
\renewcommand{\arraystretch}{0.9}
\setlength\extrarowheight{0pt}
\resizebox{\linewidth}{!}{%
\begin{tabular}{c|c|c|c}
\toprule
Group & small & medium & large \\
\midrule
Argmax-prob & +5.47 \% & +2.92 \% & +2.01 \% \\
CIA-RankSEG & +1.19 \% & +0.82 \% & +0.42 \% \\
\bottomrule
\end{tabular}
}
\vspace{-1em}
\end{wraptable}

\textit{Small vs. Large categories:} We divide ADE20K categories into small, medium, and large groups of equal size based on their average pixel area. As shown in \cref{tab:small_large}, we observe a similar trend: improvements are more pronounced on smaller categories, which are typically harder to segment and benefit more from modeling label dependence.

\vspace{-0.5em}
\subsection{Statistical Significance of Improvements}
\label{sec:stat_significance}
\begin{figure}[t]
\vspace{-1.2em}
\centering
\begin{subfigure}{0.32\textwidth}
    \centering
    \includegraphics[width=\linewidth]{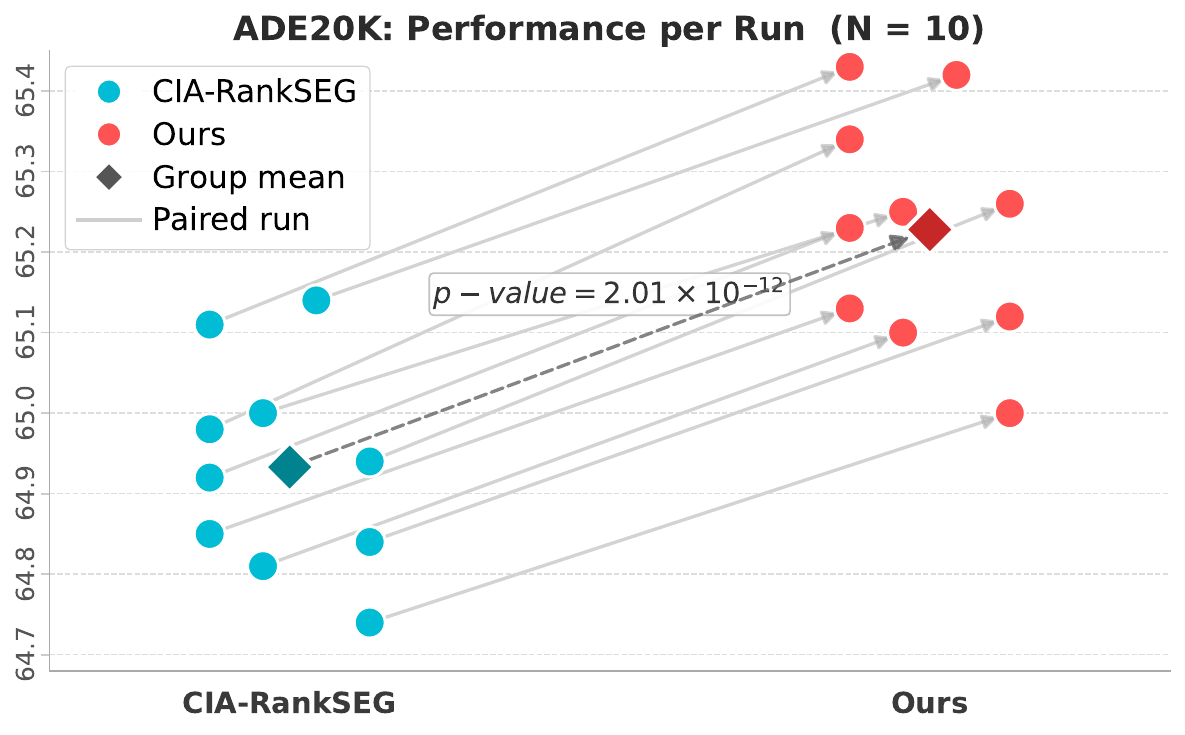}
    \label{fig:per_run_ade20k}
\end{subfigure}
\hfill
\begin{subfigure}{0.32\textwidth}
    \centering
    \includegraphics[width=\linewidth]{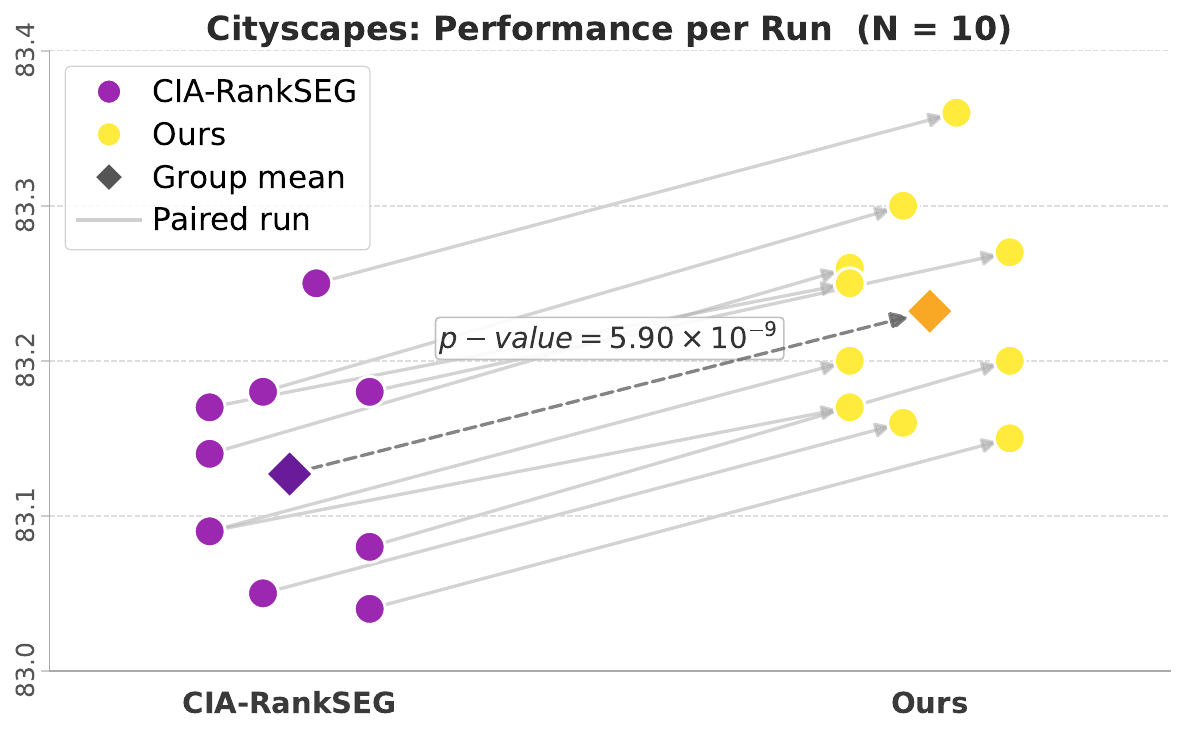}
    \label{fig:per_run_cityscapes}
\end{subfigure}
\hfill
\begin{subfigure}{0.32\textwidth}
    \centering
    \includegraphics[width=\linewidth]{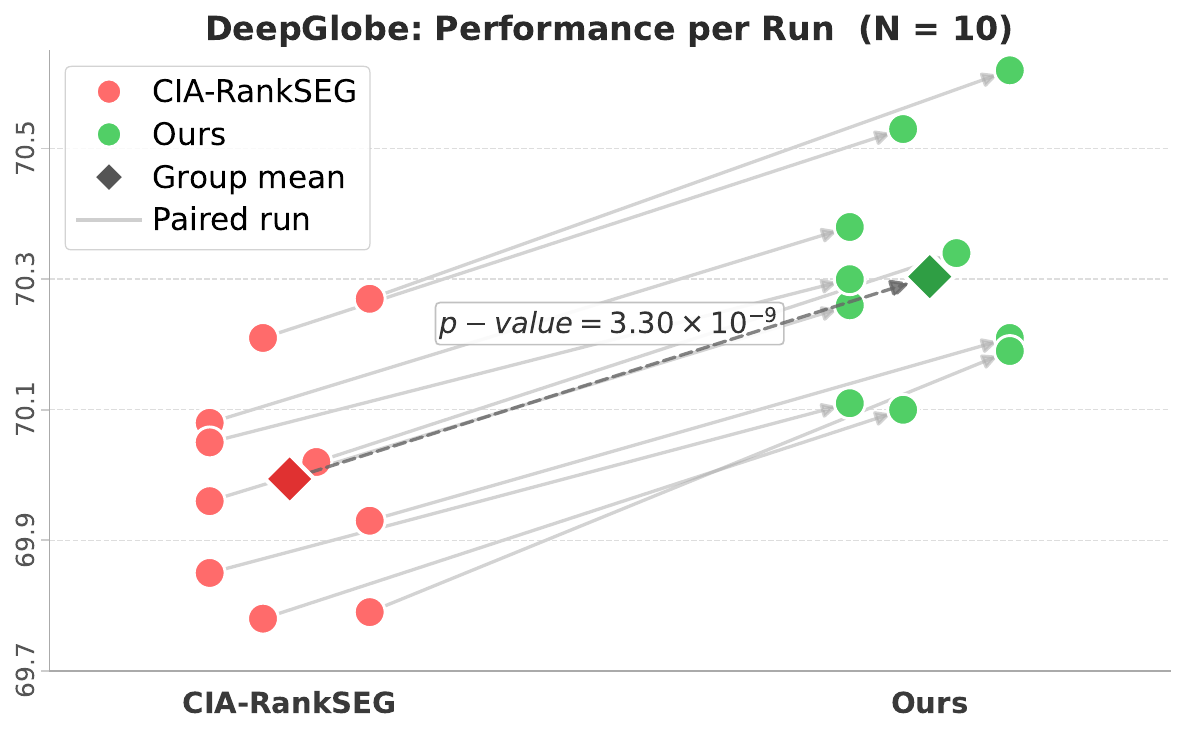}
    \label{fig:per_run_deepglobe}
\end{subfigure}
\vspace{-1.2em}
\caption{Our method consistently outperforms CIA-RankSEG across 10 retraining runs on (a) ADE20K, (b) Cityscapes, and (c) DeepGlobe, with paired $t$-test $p$-values $\ll 0.01$ confirming the statistical significance.}
\label{fig:per_run}
\vspace{-1.2em}
\end{figure}

We conduct 10 independent retraining runs with different random seeds on ADE20K, Cityscapes, and DeepGlobe using UPerNet, and report per-run performance comparison between CIA-RankSEG and our method in \cref{fig:per_run}. Each plot also reports the p-value of a paired t-test, which is well below 0.01, confirming the statistical significance of the improvements. Furthermore, \textbf{our method consistently outperforms CIA-RankSEG across all runs, reinforcing its role as an effective training-free inference module that squeezes additional performance from pre-trained networks.}

\subsection{Runtime Comparison}
\vspace{-0.3em}

To assess practicality, we compare the runtime of our method with CIA-RankSEG, using DeepLabV3+ on LiTS/KiTS and UPerNet on the remaining datasets. All experiments run on a single NVIDIA A100 GPU, with runtime measured over the entire inference process on the test set. As shown in \cref{fig:runtime}, the model forward pass (bottom bar) dominates the overall inference time, and the additional cost introduced by our method is marginal relative to CIA-RankSEG. This efficiency stems from a careful algorithmic design that reduces the complexity to $\mcal{O}(d \log d)$, matching CIA-RankSEG up to a slightly larger constant factor arising from the additional convolution steps. These results confirm that our method serves as a practical module that can be readily integrated into existing segmentation pipelines without incurring significant computational overhead.

\vspace{-0.4em}
\subsection{Parameter Robustness in SLD}
\vspace{-0.4em}

Our method introduces a single hyperparameter, the SLD kernel bandwidth $\theta$, which controls the decay rate of label covariance. We perform a sensitivity analysis by sweeping $\theta$ from $3$ to $600$ on KiTS and LiTS with DeepLabV3+. As shown in \cref{fig:theta}, the curve exhibits two distinct regimes. \textbf{(a) Robust plateau ($100 \le \theta \le 600$):} performance is nearly invariant across this range, indicating that our method requires no careful tuning and delivers consistent gains out of the box. \textbf{(b) Baseline convergence ($\theta < 100$):} as $\theta \to 0$, performance degrades toward the CIA-RankSEG baseline. This behavior directly follows from \eqref{eq:sld}: as $\theta \to 0$, the covariance vanishes for all distinct label pairs, reducing SLD to CIA. This further validates our theoretical motivation of our method, and serves as an implicit ablation, underscoring the importance of incorporating label dependence.

\begin{figure}[h]
\centering
\begin{minipage}{0.47\textwidth}
    \centering
    \includegraphics[width=\linewidth]{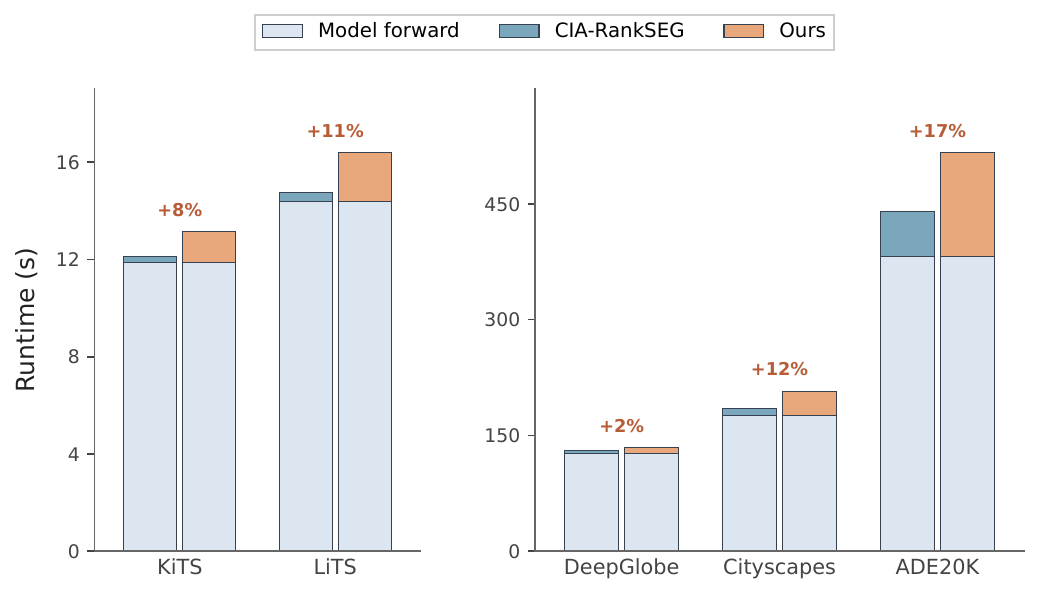}
    \vspace{-0.5em}
    \caption{Runtime comparison between CIA-RankSEG and our method during the inference. The additional cost of our method is minor compared to the overall inference time.}
    \label{fig:runtime}
\end{minipage}
\hfill
\begin{minipage}{0.47\textwidth}
    \centering
    \includegraphics[width=\linewidth]{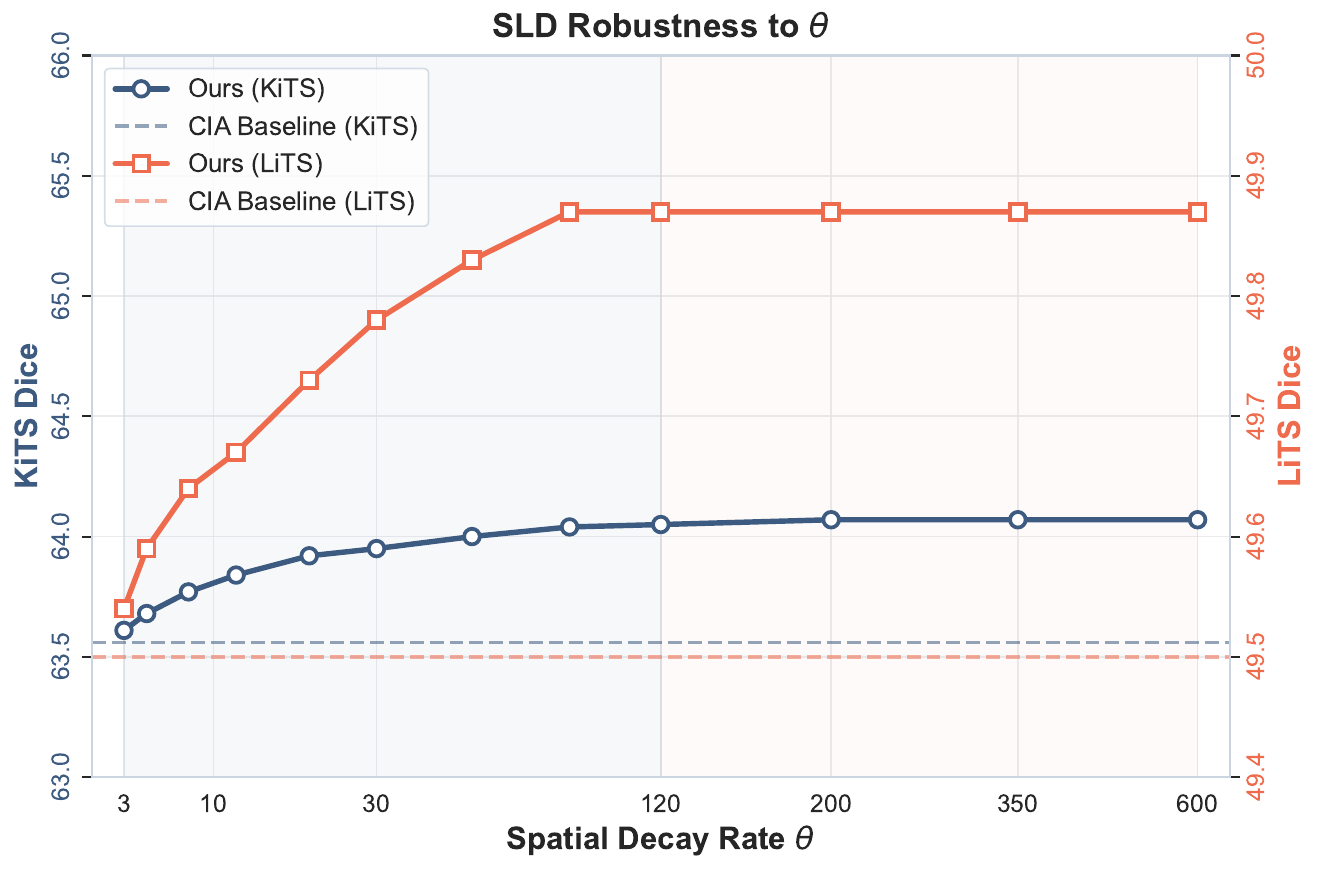}
    \vspace{-0.5em}
    \caption{Robustness to $\theta$ in SLD. The performance is stable and saturates as $\theta \ge 100$.}
    \label{fig:theta}
\end{minipage}
\vspace{-0.5em}
\end{figure}
\vspace{-0.5em}

\vspace{-0.5em}
\section{Conclusion}
\label{sec:conclusion}
\vspace{-0.5em}
In this work, we addressed a critical limitation in image segmentation by relaxing the CIA to SLD modeling. We incorporated local label correlations into RankSEG while developing an efficient algorithm for optimizing Dice/IoU with practical $\mcal{O}(d \log d)$ complexity. Extensive experiments across multiple datasets demonstrate that our method consistently outperforms CIA-based methods, especially on challenging cases where image information is noisy and label dependence is critical. One limitation of our work is that the proposed SLD captures label dependence purely through spatial structure, without exploiting appearance cues such as similarities between $X_i$ and $X_j$ or between $p_i$ and $p_j$. Incorporating such cues is highly non-trivial, as it requires careful redesign to effectively capture the complex interplay between appearance and spatial structure, while preserving computational efficiency. We leave this as a promising direction for future work.

\section{Acknowledgements}
We thank the anonymous Area Chair and reviewers for their valuable feedback, suggestions and
support.
This work was supported by the CUHK startup grant and the CUHK Faculty of Science Direct Grant.

\bibliography{ref}
\bibliographystyle{unsrtnat}


\newpage
\appendix
\crefalias{section}{appendix}
\crefalias{subsection}{subappendix}

\addcontentsline{toc}{section}{Appendix} 
\part{} 
\parttoc 

\section{Proofs}

\subsection{Proof of \cref{lem:fixed}}


\begin{proof}
  From \eqref{eq:tau_opt}, $\tau^*$ can be expressed as:
  \begin{align*}
    \tau^* = \textstyle{\argmax_{\tau \in [d]} \max_{\bm{l} \in \rho_{[d]}} \sum_{j=1}^{\tau} M_{l_j, \tau}},
  \end{align*}
  where $\rho_{[d]}$ denotes the set of all permutations of $[d]$. Since $\bm{o}(\tau^*) = {argsort}(\bm{M}_{1:d, \tau^*})$ already maximizes the inner sum for $\tau^*$, so we have
  \begin{align*}
    \sum_{j=1}^{\tau^*} M_{o_j(\tau^*), \tau^*} = \textstyle{\max_{\tau \in [d]} \max_{\bm{l} \in \rho_{[d]}} \sum_{j=1}^{\tau} M_{l_j, \tau}} \ge \textstyle{\max_{\tau \in [d]}} \sum_{j=1}^{\tau} M_{o_j(\tau^*), \tau}
  \end{align*}
  The inequality is due to replacing the inner maximizer with a specific choice $\bm{o}(\tau^*)$. The right-hand side is precisely what is maximized in $T$, and thus $T(\tau^*) = \tau^*$.
\end{proof}

\subsection{Proof of \cref{lem:ascent}}


\begin{proof}
For any $\tau \in [d]$, we can establish the following chain of inequalities:
  \begin{equation*}
    \begin{aligned}
      \Phi(T(\tau)) &= \pi_{T(\tau)}\big(\bm{o}(T(\tau))\big) & \left( = \sum_{j=1}^{\textcolor{green!50!black}{T(\tau)}} M_{o_j(\textcolor{green!50!black}{T(\tau)}), \textcolor{green!50!black}{T(\tau)}} \right) \\
      &\geq \pi_{T(\tau)}\big(\bm{o}(\tau)\big) & \left( = \sum_{j=1}^{\textcolor{green!50!black}{T(\tau)}} M_{o_j(\textcolor{red}{\tau}), \textcolor{green!50!black}{T(\tau)}} \right) \\
      &\geq \pi_{\tau}\big(\bm{o}(\tau)\big) & \left( = \sum_{j=1}^{\textcolor{red}{\tau}} M_{o_j(\textcolor{red}{\tau}), \textcolor{red}{\tau}} \right) \\
      &= \Phi(\tau).
    \end{aligned}
  \end{equation*}
  The first inequality holds because 
  $\bm{o}(T(\tau)) = \text{argsort}(\bm{M}_{1:d,T(\tau)})$ maximizes the 
  top-$T(\tau)$ partial sum over all permutations, and the second 
  follows from $T(\tau) = \argmax_{\tau'} \pi_{\tau'}(\bm{o}(\tau))$. 
  This establishes claim~(a).

  For claim~(b), if the maximizer in $T$ is always unique and $T(\tau) \neq \tau$, the second inequality above is strict, giving $\Phi(T(\tau)) > \Phi(\tau)$.
  Applied along the iterates, this means $\Phi(\tau^{(t+1)}) > \Phi(\tau^{(t)})$ whenever $T(\tau^{(t)}) \neq \tau^{(t)}$. If this strict inequality were to hold for all $t \in \{0,1,\dots,d\}$, we would obtain $d+1$ pairwise distinct values $\Phi(\tau^{(0)}) < \Phi(\tau^{(1)}) < \cdots < \Phi(\tau^{(d)})$, contradicting the fact that $\Phi$ takes at most $d$ values on the finite set $[d]$. Hence there exists a smallest index $t^\star \leq d$ with $T(\tau^{(t^\star)}) = \tau^{(t^\star)}$. Setting $\tau^\star := \tau^{(t^\star)}$, a trivial induction gives $\tau^{(t^\star + s)} = \tau^\star$ for all $s \geq 0$, so the iterates reach the fixed point $\tau^\star$ of $T$ in at most $d$ steps.
\end{proof}

\subsection{Proof of \cref{lem:global_ranking}}


\begin{proof}
  Recall the definition $M_{j,\tau} = p_j / (\tau + \mu_j)$. For any $j \neq j^{\prime}$ such that $p_j \ge p_{j^{\prime}}$, we have $\mu_j \le \mu_{j^{\prime}}$ by assumption, and thus
  \begin{align*}
    M_{j,\tau} = \frac{p_j}{\tau + \mu_j} \ge \frac{p_{j^{\prime}}}{\tau + \mu_{j^{\prime}}} = M_{j^{\prime}, \tau} \quad \text{for all } \tau.
  \end{align*}
  Hence, the ranking of $\bm{M}_{1:d, \tau}$ is exactly identical to that of $\bm{p}$ for all $\tau$. In particular, $\bm{o}(\tau) = \text{argsort}(\bm{M}_{1:d, \tau})=\text{argsort}(\bm{p})$ for all $\tau$. Since we initialize with $\bm{o}(\tau^{(0)}) = \text{argsort}(\bm{p})$, the iteration converges in one step.
\end{proof}

\subsection{Proof of \cref{thm:error}}

\begin{proof}
  Without loss of generality, let the image be an $L \times L$ grid where $d = L^2$. Each pixel index $i \in [d]$ maps to a 2D coordinate $(u_i, v_i) \in [L]^2$ given by:
  \begin{align*}
    u_i = \lfloor (i-1)/L \rfloor + 1, \quad v_i = ((i-1)\text{ mod } L) + 1.
  \end{align*}
  The squared Euclidean distance between pixels $i$ and $k$ is then given by:
  \begin{align*}
    r(i,k)^2 = (u_i - u_k)^2 + (v_i - v_k)^2.
  \end{align*}

  \textbf{Step 1: Bound the unconditional variance.} Denote $\Gamma = \sum_{i=1}^d Y_i$. Its variance under SLD is:
  \begin{align*}
    \text{Var}(\Gamma) = \sum_{i=1}^{d} \sum_{k=1}^d \Sigma_{ik}, \quad \text{where} \quad \Sigma_{ik} = \sqrt{\Sigma_{ii}\Sigma_{kk}} \exp\left( -\frac{r(i,k)^2}{2\theta^2} \right).
  \end{align*}
  Note $\Sigma_{ii} = p_i(1-p_i) \le 1/4$ for all $i$. It follows that
  \begin{align*}
    \Sigma_{ik} \le 0.25 \exp\left(-\frac{(u_i - u_k)^2}{2\theta^2} \right) \exp\left(-\frac{(v_i - v_k)^2}{2\theta^2} \right).
  \end{align*}
  For a fixed pixel $i$, we can extend the summation boundaries from the finite $L \times L$ grid to the infinite $\mathbb{Z}^2$ grid, and obtain
  \begin{align*}
    \sum_{k=1}^d \Sigma_{ik} \le 0.25 \left( \sum_{\Delta u=-\infty}^{\infty} \exp\left( -\frac{\Delta u^2}{2\theta^2} \right) \right) \left( \sum_{\Delta v=-\infty}^{\infty} \exp\left( -\frac{\Delta v^2}{2\theta^2} \right) \right).
  \end{align*}
  Consider $\sum_{k=-\infty}^{\infty} \exp(-k^2/2\theta^2)$, since the terms decay exponentially, the series converges to a finite constant that depends only on $\theta$, denoted as $C_{\theta}$. Hence $\sum_{k=1}^d \Sigma_{ik} \le 0.25 C_{\theta}^2 =: K_{\theta}$ for all $i$, and therefore
  \begin{align*}
    \text{Var}(\Gamma) = \sum_{i=1}^d \sum_{k=1}^d \Sigma_{ik} \le \sum_{i=1}^d K_{\theta} = K_{\theta} d = \mcal{O}(d).
  \end{align*}
  \textbf{Step 2: Bound the conditional variance.} We decompose $\text{Var}(\Gamma)$ conditioning on $Y_j$ by the law of total variance:
  \begin{align*}
    \text{Var}(\Gamma) = \mbb{E}_{Y_j}\left[ \text{Var}(\Gamma \mid Y_j) \right] + \text{Var}_{Y_j}\left[ \mbb{E}(\Gamma \mid Y_j) \right].
  \end{align*}
  Dropping the second term gives
  \begin{align*}
    & p_j \text{Var}(\Gamma \mid Y_j=1) + (1-p_j) \text{Var}(\Gamma \mid Y_j=0) = \mbb{E}_{Y_j}\left[ \text{Var}(\Gamma \mid Y_j) \right] \le \text{Var}(\Gamma) = \mcal{O}(d).
  \end{align*}
  Since $p_j \ge c > 0$ for some constant $c$, we have $\text{Var}(\Gamma \mid Y_j=1) = \mcal{O}(d)$. This proves claim~(a).

  \textbf{Step 3: Bound the RMA error.} From \cref{thm:rma}, the approximation error of RMA is bounded by $\mcal{E} \le \sigma_j^2 / [\tau(\mu_j(\mu_j+\tau) + \sigma_j^2)]$. Note that $|M_{j,\tau}-S_{j,\tau}| \le p_j \mcal{E}$, and $p_j \le 1$. Under the additional assumption $\mu_j = \Theta(d)$, we have:
  \begin{align*}
    |M_{j,\tau} - S_{j,\tau}| \le \mcal{E} \le \frac{\sigma_j^2}{\tau[\mu_j(\mu_j+\tau) + \sigma_j^2]} = \frac{\mcal{O}(d)}{\tau[\Theta(d) (\Theta(d) + \tau) + \mcal{O}(d)]} \le \frac{\mcal{O}(d)}{\Theta(d^2)} = \mcal{O}(d^{-1}).
  \end{align*}
  This proves claim~(b).
\end{proof}

\subsection{Proof of \cref{thm:rankseg} and \cref{thm:rma}}
\cref{thm:rankseg} combines results from \citet{dembczynski2013optimizing} and \citet{dai2023rankseg}, and \cref{thm:rma} is a direct application of \citet[Theorem 1]{wooff1985bounds}. For completeness, we provide the proof of \cref{thm:rankseg} here, while deferring the proof of \cref{thm:rma} to the original paper.


\begin{proof}
  Recall that the expected Dice score for a segmentation mask $\bm{\delta} \in \{0,1\}^d$ is defined as:
  \begin{align*}
    \mbb{E}[\Dice(\bm{\delta})] = \mbb{E}\left[ \frac{2 \bm{\delta}^{\intercal} \bm{Y}}{\| \bm{\delta} \|_1 + \| \bm{Y} \|_1} \right] = \mbb{E}\left[ \frac{2 \sum_{j=1}^d \delta_j Y_j}{\| \bm{\delta} \|_1 + \| \bm{Y} \|_1} \right].
  \end{align*}
  Let $\tau = \| \bm{\delta} \|_1$ denote the predicted volume. The expected Dice score can then be rewritten as:
  \begin{align*}
    \mbb{E}[\Dice(\bm{\delta})] = 2\sum_{j=1}^d \delta_j \mbb{E} \left[ \frac{Y_j}{\tau + \Gamma_j} \right] = 2\sum_{j=1}^d \delta_j p_j \mbb{E} \left[ \frac{1}{\tau + \Gamma_j} \mid Y_j = 1 \right] = 2 \sum_{j=1}^d \delta_j p_j \mbb{E} (\tau + \Gamma_j)^{-1}.
  \end{align*}
  Denote the contribution of pixel $j$ for a fixed volume $\tau$ as $S_{j,\tau} = p_j \mbb{E} (\tau + \Gamma_j)^{-1}$. We decompose the maximization into bi-level optimization:
  \begin{align*}
    \textstyle{\max_{\tau \in [d]} \max_{\{\bm{\delta} \in \{0,1\}^d: \|\bm{\delta}\|_1 = \tau\}}} 2\sum_{j=1}^d \delta_j S_{j,\tau}.
  \end{align*}
  For the inner maximization with $\tau$ fixed, the optimal strategy is to set $\delta_j = 1$ for the top-$\tau$ pixels with the largest $S_{j,\tau}$, yielding $\omega_{\tau} = \sum_{j \in \text{Top}_{\tau}(\bm{S}_{1:d, \tau})} S_{j,\tau}$. For the outer maximization, we evaluate $\omega_{\tau}$ for all $\tau$ and select the best one. This completes the proof of \cref{thm:rankseg}.
\end{proof}

\section{Two-Dimensional Convolutional Form of \texorpdfstring{$\bm{\mu}$}{mu}}
\label{sec:conv2d}

In \cref{sec:rma}, we express $\bm{\mu}$ through a one-dimensional convolution over the vectorized labels for notational simplicity, where the kernel entry $\bm{\kappa}_{d-j+i} = \mcal{K}(|i-j|)$ is indexed by the offset between the 1D indices $i$ and $j$. For images, however, the distance $r(i,j)$ in \cref{asm:local} is the spatial distance between two pixels on the 2D grid, which is not a function of $|i-j|$ after vectorization. In this section, we present the rigorous two-dimensional counterpart, which is what our implementation computes.

\paragraph{Setup.} Consider an image of height $H$ and width $W$, so that $d = HW$. Following the notation in the proof of \cref{thm:error}, each pixel index $j \in [d]$ maps to a 2D coordinate $(u_j, v_j) \in [H] \times [W]$ under the row-major ordering:
\begin{align*}
  u_j = \lfloor (j-1)/W \rfloor + 1, \quad v_j = ((j-1)\text{ mod } W) + 1,
\end{align*}
and the distance in \cref{asm:local} is the Euclidean distance on the grid, i.e., $r(i,j) = \sqrt{(u_i - u_j)^2 + (v_i - v_j)^2}$. Accordingly, we reshape $\bm{p}$, $\bm{\nu} = \sqrt{\text{diag}(\bm{\Sigma})}$, and $\bm{\mu}$ into matrices $\bm{P}, \bm{V}, \bm{U} \in \mbb{R}^{H \times W}$, with $P_{u_j, v_j} = p_j$, $V_{u_j, v_j} = \sqrt{\Sigma_{jj}}$, and $U_{u_j, v_j} = \mu_j$. Throughout this section, $(u, v) \in [H] \times [W]$ denotes a generic 2D coordinate, while $i, j \in [d]$ are reserved for 1D pixel indices.

\paragraph{2D convolutional form.} Define the 2D kernel $\bm{\mcal{K}} \in \mbb{R}^{(2H-1) \times (2W-1)}$ that stores the kernel values at all possible displacements $(\Delta u, \Delta v)$:
\begin{align*}
  \bm{\mcal{K}}_{H + \Delta u, \, W + \Delta v} = \mcal{K}\Big( \sqrt{\Delta u^2 + \Delta v^2} \Big), \quad |\Delta u| \le H-1, \ |\Delta v| \le W-1,
\end{align*}
so that the center entry $\bm{\mcal{K}}_{H,W} = \mcal{K}(0) = 1$ corresponds to zero displacement. Then the SLD correction term in $\mu_j$ can be written as a 2D convolution (cross-correlation) of $\bm{V}$ with $\bm{\mcal{K}}$:
\begin{align*}
  \sum_{i=1}^d \sqrt{\Sigma_{ii}} \, \mcal{K}(r(i,j))
  &= \sum_{u=1}^{H} \sum_{v=1}^{W} V_{u,v} \, \mcal{K}\Big( \sqrt{(u - u_j)^2 + (v - v_j)^2} \Big) \\
  &= \sum_{u=1}^{H} \sum_{v=1}^{W} V_{u,v} \, \bm{\mcal{K}}_{H + u - u_j, \, W + v - v_j}
  =: (\bm{V} \star \bm{\mcal{K}})_{u_j, v_j}.
\end{align*}
Therefore, computing all $j \in [d]$ simultaneously yields the matrix form
\begin{align}
  \bm{U} = q \cdot \bm{1}_{H \times W} + \frac{\bm{V}}{\bm{P}} \odot (\bm{V} \star \bm{\mcal{K}}), \label{eq:mu_2d}
\end{align}
where $q = \sum_{j=1}^d p_j$, $\bm{1}_{H \times W}$ is the all-ones matrix, and both the division and $\odot$ are element-wise. \eqref{eq:mu_2d} is the exact 2D analogue of the 1D expression in \cref{sec:rma}: the vector $\bm{\kappa} \in \mbb{R}^{2d-1}$ is replaced by the matrix $\bm{\mcal{K}} \in \mbb{R}^{(2H-1) \times (2W-1)}$, and the 1D offset $|i-j|$ is replaced by the 2D displacement $(u_i - u_j, v_i - v_j)$.

\paragraph{Computational complexity.} The 2D convolution $\bm{V} \star \bm{\mcal{K}}$ can be evaluated via the 2D FFT: after zero-padding $\bm{V}$ and $\bm{\mcal{K}}$ to a common size of $\mcal{O}(H) \times \mcal{O}(W)$ to avoid circular wrap-around, the cost is $\mcal{O}(HW \log(HW)) = \mcal{O}(d \log d)$, matching the complexity stated in \cref{sec:rma}. The remaining element-wise operations in \eqref{eq:mu_2d} cost $\mcal{O}(d)$.

Moreover, the Gaussian kernel is separable: since $\mcal{K}(\sqrt{\Delta u^2 + \Delta v^2}) = \exp(-\Delta u^2/2\theta^2) \exp(-\Delta v^2/2\theta^2) = \mcal{K}(|\Delta u|) \mcal{K}(|\Delta v|)$, the 2D kernel is rank-one, $\bm{\mcal{K}} = \bm{\kappa}_H \bm{\kappa}_W^{\intercal}$, where
\begin{gather*}
  \bm{\kappa}_H = (\mcal{K}(H-1), \cdots, \mcal{K}(0), \cdots, \mcal{K}(H-1))^{\intercal} \in \mbb{R}^{2H-1}, \\
  \bm{\kappa}_W = (\mcal{K}(W-1), \cdots, \mcal{K}(0), \cdots, \mcal{K}(W-1))^{\intercal} \in \mbb{R}^{2W-1}.
\end{gather*}
Consequently, the 2D convolution decomposes into two 1D convolutions of the same form as in \cref{sec:rma}: first convolving each column of $\bm{V}$ with $\bm{\kappa}_H$, and then convolving each row of the result with $\bm{\kappa}_W$, i.e.,
\begin{align*}
  (\bm{V} \star \bm{\mcal{K}})_{u_j, v_j} = \sum_{v=1}^{W} \mcal{K}(|v - v_j|) \sum_{u=1}^{H} \mcal{K}(|u - u_j|) \, V_{u,v}.
\end{align*}
Using the 1D FFT for each row and column, the total cost is $\mcal{O}(HW (\log H + \log W)) = \mcal{O}(d \log d)$. This also clarifies the role of the 1D notation in \cref{sec:rma}: it is exactly the building block of the separable 2D convolution. The same argument extends directly to 3D volumes (e.g., $H \times W \times D$ voxels) with a separable 3D Gaussian kernel, again with $\mcal{O}(d \log d)$ complexity.

\section{Efficient Cumulative Sums in Fixed-Point Iteration via Taylor Approximation}
\label{sec:cumsum}

We turn to the efficient implementation of the fixed-point iteration described in \eqref{eq:T}, focusing on the computation after $\bm{o}(\tau)$ is obtained. Without loss of generality, we regard the ranking $\bm{o}(\tau)$ as the natural ordering and simplify the notation to $\pi_{\tau} = \sum_{j=1}^{\tau} p_j / (\tau + \mu_j)$. A direct implementation requires $\mcal{O}(d)$ time for each $\tau$, resulting in $\mcal{O}(d^2)$ total complexity. The core inefficiency arises because the denominator tightly couples $\tau$ and $j$, preventing the reuse of intermediate quantities. To decouple $\tau$ from $j$, we consider the Taylor expansion of $x \mapsto (\tau + x)^{-1}$ around $\bar{\mu}_{\tau} = \frac{1}{\tau} \sum_{j=1}^{\tau} \mu_{j}$:
\begin{align*}
  \frac{1}{\tau + \mu_j} = \frac{1}{\tau + \bar{\mu}_{\tau}} - \frac{\mu_j - \bar{\mu}_{\tau}}{(\tau + \bar{\mu}_{\tau})^2} + \frac{(\mu_j - \bar{\mu}_{\tau})^2}{(\tau + \bar{\mu}_{\tau})^3} + \mcal{O}\left( \frac{(\mu_j - \bar{\mu}_{\tau})^3}{(\tau + \bar{\mu}_{\tau})^4} \right).
\end{align*}
Truncating at the second-order term is an empirical trade-off between approximation accuracy and numerical stability. This yields our final approximation
\begin{align*}
  \pi_{\tau} = \sum_{j=1}^{\tau} \frac{p_j}{\tau + \mu_j} \approx \tilde{\pi}_{\tau} &= \sum_{j=1}^{\tau} p_j \left( \frac{1}{\tau + \bar{\mu}_{\tau}} - \frac{\mu_j - \bar{\mu}_{\tau}}{(\tau + \bar{\mu}_{\tau})^2} + \frac{(\mu_j - \bar{\mu}_{\tau})^2}{(\tau + \bar{\mu}_{\tau})^3} \right) \\
  &= \frac{Z_{\tau}^{0}}{\tau + \bar{\mu}_{\tau}} - \frac{Z_{\tau}^{1} - \bar{\mu}_{\tau}Z_{\tau}^{0}}{(\tau + \bar{\mu}_{\tau})^2} + \frac{Z_{\tau}^{2} - 2\bar{\mu}_{\tau}Z_{\tau}^{1} + \bar{\mu}_{\tau}^2Z_{\tau}^{0}}{(\tau + \bar{\mu}_{\tau})^3}, \\
  \text{where} \quad & Z_{\tau}^{0} = \sum_{j=1}^{\tau} p_j, \quad Z_{\tau}^{1} = \sum_{j=1}^{\tau} p_j \mu_j, \quad Z_{\tau}^{2} = \sum_{j=1}^{\tau}p_j \mu_j^2.
\end{align*}
$\bm{Z}_{1:d}^{0}$, $\bm{Z}_{1:d}^{1}$, and $\bm{Z}_{1:d}^{2}$ can be computed in $\mcal{O}(d)$ via cumulative sums, and $\tilde{\bm{\pi}}_{1:d}$ is obtained in $\mcal{O}(d)$.

\section{Convergence of the Fixed-Point Optimization}
\label{sec:convergence}

In \cref{sec:fixed-point}, we established two theoretical guarantees for the fixed-point iteration \(\tau^{(t+1)} = T(\tau^{(t)})\): under the uniqueness condition for the maximizer of \(T\), the iterates monotonically increase the objective \(\Phi\) and converge to a fixed point (\cref{lem:ascent}); moreover, when the label distribution satisfies certain conditions, convergence to the global optimum \(\tau^{*}\) is achieved in a single step (\cref{lem:global_ranking}). Although these conditions are difficult to verify analytically on real data, the empirical results in this section show that the iteration behaves in close accordance with these guarantees: it converges within a few steps and reaches a point that is essentially near the global optimum. Specifically, we address two questions: (Q1) How many iterations are required for convergence in practice? (Q2) How close is the limit point to the global optimum? The findings below provide strong support for the practical effectiveness of the algorithm and offer empirical validation of \cref{lem:ascent,lem:global_ranking}.

\subsection{Convergence within Three Steps}

\cref{tab:convergence} reports the distribution of the number of iterations required for convergence across all test samples. For multi-class datasets, we average the number of steps over classes for each sample. Across all five benchmarks, the iteration terminates within $2.5$ steps for over $99.8\%$ of samples, and the fraction of samples requiring more than $2.5$ steps never exceeds $0.22\%$.

\begin{table}[h]
\centering
\caption{Distribution of the number of iterations required for convergence across all test samples. On every dataset, the iteration terminates within $2.5$ steps for the vast majority of samples, with only a negligible fraction requiring more.}
\label{tab:convergence}
\begin{tabular}{lcccccc}
\toprule
Dataset & LiTS & KiTS & ADE20K & Cityscapes & DeepGlobe \\
\midrule
\#Steps $< 1.5$ & 63.36\% & 88.52\% & 9.30\% & 2.80\% & 40.49\% \\
$1.5 \le$ \#Steps $< 2.5$ & 36.42\% & 11.41\% & 90.55\% & 97.20\% & 50.51\% \\
$2.5 \le$ \#Steps & 0.22\% & 0.07\% & 0.15\% & 0.00\% & 0.00\% \\
\bottomrule
\end{tabular}
\end{table}

These results indicate that the fixed-point optimization converges substantially faster than the worst-case bound, allowing us to safely conclude that the proposed algorithm runs in $\mcal{O}(d \log d)$ time in practice. We attribute this fast convergence to two factors: (i) the initialization $\tau^{(0)}$ induced by the probabilities $\bm{p}$ is already close to the optimal $\tau^*$, and (ii) the landscape of $T$ is well-behaved, with a large basin of attraction around $\tau^*$.

\subsection{Convergent Point is Near-Optimal}

To assess whether the iteration locates the global optimum, we conduct a brute-force search on LiTS and KiTS, for which exhaustive enumeration is computationally affordable, and compare the performance between the global optimum and the convergent point of the fixed-point optimization.

\begin{table}[h]
\centering
\caption{Comparison between the convergent point obtained by the fixed-point optimization and the global optimum obtained by brute-force search. The performance gap is at most $0.03$ on both metrics, indicating that the iteration consistently recovers a near-optimal fixed point.}
\label{tab:optimality}
\begin{tabular}{l|c|c|c|c}
\toprule
Dataset & \multicolumn{2}{c|}{KiTS} & \multicolumn{2}{c}{LiTS} \\
\midrule
Metric & IoU & Dice & IoU & Dice \\
\midrule
Fixed-point Optimization & 56.70 & 64.07 & 40.48 & 49.88 \\
Brute-force Search & 56.73 & 64.10 & 40.48 & 49.88 \\
\bottomrule
\end{tabular}
\end{table}

As shown in \cref{tab:optimality}, the gap between the fixed-point solution and the global optimum is at most $0.03$ across both datasets and both metrics, and is essentially zero on LiTS. These results suggest that the fixed-point optimization not only converges rapidly but also reliably identifies a solution that is very close to the best achievable performance. This near-optimality can again be attributed to the favorable initialization and the well-behaved landscape of $T$~\citep{tarski1955lattice,ortega2000iterative}.

\section{Discussion on SLD and CRF}
\label{sec:crf}

In ~\cref{sec:sld}, we introduced the SLD model as a more realistic alternative to the CIA for modeling label distributions in image segmentation. A natural question is how SLD relates to the widely used Conditional Random Field (CRF) model~\citep{lafferty2001conditional,krahenbuhl2011efficient}, which also captures local label correlations and acts as a post-processing step following the network outputs~\citep{chen2014semantic}. In this section, we clarify the distinction between these two models and explain why SLD is more suitable for our purpose.

\subsection{The Conditional Random Field (CRF) Model}

CRFs have been a cornerstone of semantic segmentation, typically used to refine pixel-wise predictions by modeling the posterior distribution $\mbb{P}(\bm{Y} \mid \bm{X})$. A CRF represents the conditional distribution as a Gibbs distribution:
\begin{align*}
  \mbb{P}(\bm{Y} \mid \bm{X}) = \frac{1}{Z(\bm{X})}\exp(-E(\bm{Y}, \bm{X})),
\end{align*}
where $E(\bm{Y}, \bm{X})$ is the energy function (we omit the dependence on $\bm{X}$ hereafter), usually decomposed into unary and pairwise potentials~\citep{krahenbuhl2011efficient}:
\begin{align*}
  E(\bm{Y}) = \sum_{i=1}^d \psi_u(Y_i) + \sum_{i<j} \psi_p(Y_i, Y_j).
\end{align*}
The unary potential $\psi_u$ is typically the negative log-probability from the network output, while the pairwise potential $\psi_p$ is designed as:
\begin{align*}
  \psi_p(Y_i, Y_j) = \mu(Y_i, Y_j) \left( w^{(1)} \underbrace{\exp\left( -\frac{|u_i - u_j|^2}{2\theta_{\alpha}^2} - \frac{|X_i-X_j|^2}{2\theta_{\beta}^2} \right)}_{\text{appearance kernel}} + w^{(2)} \underbrace{\exp\left( -\frac{|u_i - u_j|^2}{2\theta_{\gamma}^2} \right)}_{\text{smoothness kernel}} \right),
\end{align*}
where $\mu(Y_i, Y_j) = \mathbbm{1}(Y_i \neq Y_j)$ measures the label compatibility, and $u_i$ and $X_i$ denote the spatial coordinate and RGB value of pixel $i$, respectively. The appearance kernel encourages nearby pixels with similar colors to share the same label, while the smoothness kernel encourages nearby pixels to share the same label regardless of their colors.

After modeling the conditional distribution, CRF-based methods typically perform Maximum a Posteriori (MAP) inference, through certain approximate inference techniques, to obtain the final segmentation mask:
\begin{align*}
  \bm{\delta}_{\text{MAP}} = \argmax_{\bm{\delta} \in \{0,1\}^d} \mbb{P}(\bm{Y} = \bm{\delta} \mid \bm{X}) = \argmin_{\bm{\delta} \in \{0,1\}^d} E(\bm{\delta}).
\end{align*}

\subsection{Key Distinction between SLD and CRF}

Although both SLD and CRF leverage Gaussian kernels to capture spatial locality, they differ fundamentally in their mathematical formulation and intent:
\begin{itemize}[leftmargin=*]
  \item \textbf{Covariance vs. Energy Modeling:} SLD explicitly specifies the \textbf{covariance} structure of the label distribution. This allows us to efficiently compute the moment of $\Gamma_j$, which is the core requirement for optimizing the Dice/IoU expectation. In contrast, CRFs specify the \textbf{joint energy}. The resulting covariance matrix of a CRF is implicitly defined and generally lacks the closed-form structure needed for efficient RankSEG integration.
  \item \textbf{MAP Inference vs. Metric Optimization:} CRF-based methods typically perform MAP inference to find the single most probable joint configuration. However, the MAP solution is not necessarily optimal for set-based metrics like Dice or IoU. SLD is designed specifically to facilitate direct metric optimization by $\max \mbb{E}[\Dice(\bm{\delta})]$.
  \item \textbf{Parameter Complexity:} SLD requires only a single decay parameter $\theta$ to control spatial dependence. CRFs typically rely on a suite of hyperparameters ($\theta_{\alpha}$, $\theta_{\beta}$, $\theta_{\gamma}$, $w^{(1)}$, and $w^{(2)}$) that often require exhaustive cross-validation to tune for specific datasets.
\end{itemize}

The diminishing returns of CRF post-processing in modern pipelines are illustrated by the evolution of the DeepLab series. While DeepLabv2~\citep{chen2014semantic} utilized CRFs to sharpen boundaries, subsequent versions~\citep{chen2017rethinking,chen2018encoder} removed them, finding that deep architectures capture sufficient context internally. 
\begin{table}[t]
  \centering
  \caption{Performance comparison of CRF, Argmax-prob, CIA-RankSEG, and our method using UPerNet.}
  \label{tab:crf}
  \begin{tabular}{lcccccccc}
    \toprule
    \multirow{2}{*}{Method} & \multicolumn{2}{c}{CRF} & \multicolumn{2}{c}{Argmax-prob} & \multicolumn{2}{c}{CIA-RankSEG} & \multicolumn{2}{c}{Ours} \\
    \cmidrule(lr){2-3} \cmidrule(lr){4-5} \cmidrule(lr){6-7} \cmidrule(lr){8-9}
    & mIoU & mDice & mIoU & mDice & mIoU & mDice & mIoU & mDice \\
    \midrule
    DeepGlobe & 58.61 & 67.40 &59.49 & 68.60 & 60.19 & 69.57 & \textbf{60.52} & \textbf{69.91} \\
    Cityscapes & 73.26 & 80.35 & 75.66 & 82.61 & 76.17 & 83.21 & \textbf{76.25} & \textbf{83.31} \\
    ADE20K & 56.05 & 62.61 & 56.94 & 63.98 & 57.67 & 64.92 & \textbf{57.94} & \textbf{65.19} \\
    \bottomrule
  \end{tabular}
\end{table}

Our empirical results in \cref{tab:crf} corroborate this: CRFs can even degrade performance compared to simple Argmax-prob. While the MAP prior was useful for refining probability masks in early vision tasks, it may actually hinder performance when applied to the highly accurate, calibrated probabilities produced by modern networks.

\section{Extension to Multi-class Segmentation}
\label{sec:multi-class}

In non-overlapping multi-class segmentation, each pixel is assigned to exactly one class; that is, we seek $\bm{\phi} \in [C]^d$, where $C$ denotes the number of classes. To extend our method to this setting, we adopt the \textit{incremental score} strategy of~\citet{wang2025rankseg}. Specifically, we first apply the RankSEG algorithm independently to each class, yielding $C$ binary masks $\bm{\delta} \in \{0,1\}^{C \times d}$. Based on these masks, we define the following three index sets:
\begin{align*}
  &\text{Pixels predicted as class } c: \quad & \mcal{I}_c^{+} &= \{j: \delta_{c,j}=1\}, \\
  &\text{Overlapping pixels:} \quad & \mcal{I}^{\text{overlap}} &= \cup_{c \neq c^{\prime}} (\mcal{I}_c^{+} \cap \mcal{I}_{c^{\prime}}^{+}), \\
  &\text{Non-overlapping pixels of class } c: \quad & \mcal{I}_c &= \mcal{I}_c^{+} \setminus \mcal{I}^{\text{overlap}}.
\end{align*}

Pixels in $\mcal{I}_c$ are safely assigned to class $c$. For each overlapping pixel $j \in \mcal{I}^{\text{overlap}}$, we assign it to the class that yields the largest incremental score:
\begin{align*}
  & \phi_j = \argmax_{c \in [C]} \Delta_{c,j}, \quad \forall j \in \mcal{I}^{\text{overlap}}, \\
  \text{where} \quad & \Delta_{c,j} = \underbrace{\left(\sum_{i \in \mcal{I}_c} \frac{p_{c,i}}{|\mcal{I}_c| + 1 + \mu_{c,i}} + \frac{p_{c,j}}{|\mcal{I}_c| + 1 + \mu_{c,j}}\right)}_{\text{score \textcolor{red}{after} adding pixel } j \text{ to class } c} - \underbrace{\sum_{i \in \mcal{I}_c} \frac{p_{c,i}}{|\mcal{I}_c| + \mu_{c,i}}}_{\text{score \textcolor{red}{before} adding pixel } j \text{ to class } c}.
\end{align*}
Here, $p_{c,i}$ and $\mu_{c,i}$ denote the probability and conditional mean of pixel $i$ for class $c$, respectively, analogous to the binary case. Since $\Delta_{c,j}$ quantifies the score gain from assigning pixel $j$ to class $c$, it is natural to assign $j$ to the class that maximizes this gain. This procedure is guaranteed to produce a valid non-overlapping segmentation mask.

\section{Training Details}
\label{sec:training}

The training settings follow~\citet{wang2023revisiting,wang2025rankseg}, and we provide the details here for completeness. All models are trained with the standard cross-entropy loss to estimate calibrated probability masks. For DeepGlobe Land, Cityscapes, and ADE20K, we adopt the AdamW optimizer with a weight decay of 0.01. The learning rate starts from $1\mathrm{e}\text{-}6$ and is linearly warmed up during the first $1\%$ of iterations to the initial learning rate of $6\mathrm{e}\text{-}5$. Subsequently, the learning rate is decayed under a ``poly'' policy with an exponent of 1. The number of warm-up iterations is 400 for Cityscapes, and 800 for ADE20K. The total number of training iterations is 20{,}000 for DeepGlobe Land, 40{,}000 for Cityscapes, and 80{,}000 for ADE20K. Data augmentation consists of (i) random scaling within the range $[0.5, 2.0]$ and (ii) random horizontal flipping with a probability of 0.5. Since ADE20K and Cityscapes provide designated validation sets, we report performance on the validation set for these two datasets. For DeepGlobe Land, we manually split the original training set into a training subset ($80\%$) and a validation subset ($20\%$), and report performance on the validation subset.

For LiTS and KiTS, we train the models using SGD with an initial learning rate of $0.01$, momentum of $0.9$, and weight decay of $0.0005$. The learning rate is decayed under a ``poly'' policy with an exponent of 0.9. The batch size is $8$, and the number of epochs is $60$. Although these two datasets are originally multi-class segmentation tasks, we convert them into binary segmentation problems by treating only the tumor as the foreground. This conversion is necessary because we compare our method with RankDice-BA, which is applicable only to binary segmentation. Moreover, since LiTS and KiTS do not include designated test sets, we employ 5-fold cross-validation to evaluate performance.

\section{Qualitative Results}


\begin{figure}[htbp]
  \centering
  \includegraphics[width=\textwidth]{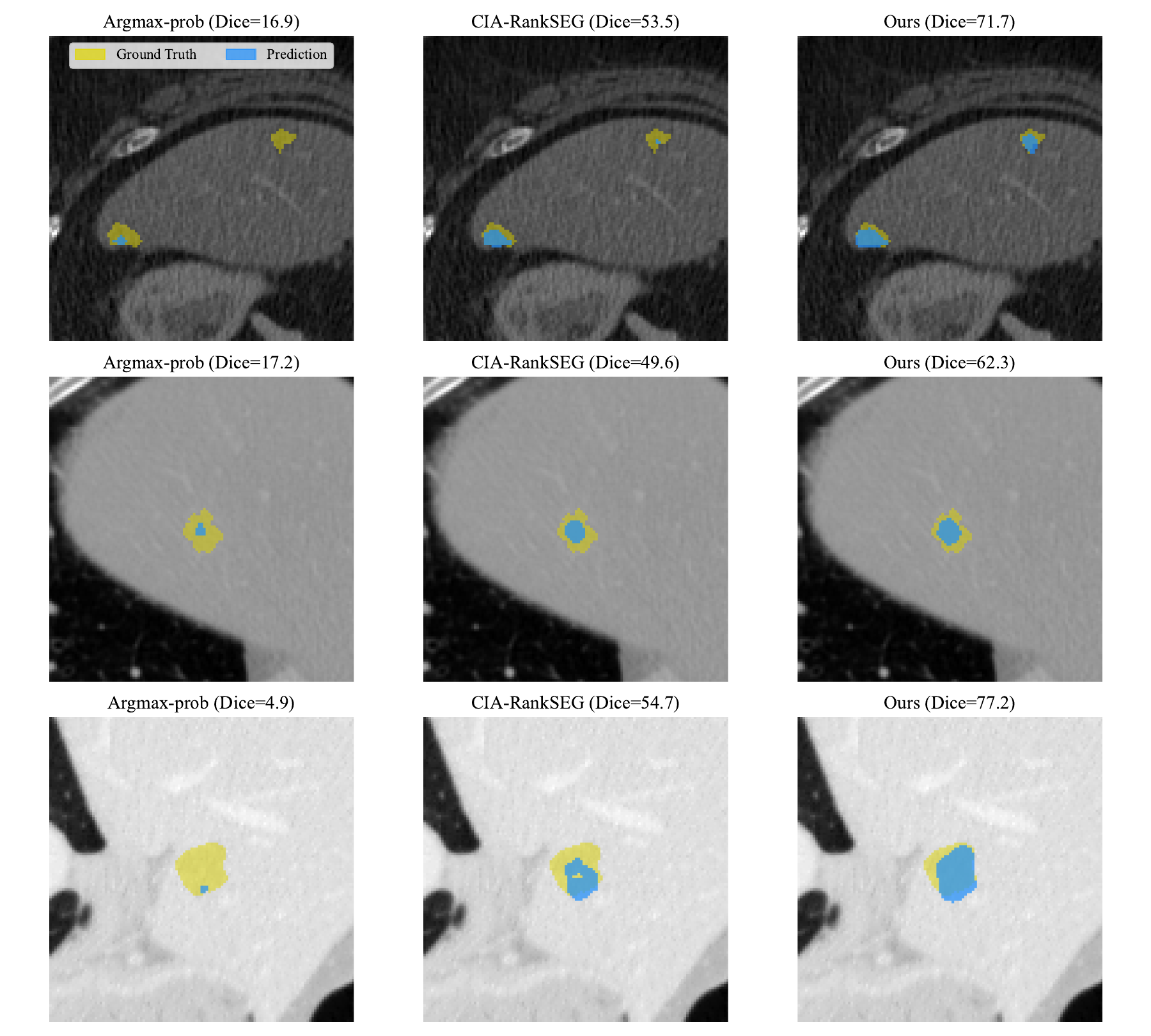}
  \caption{Qualitative comparison of segmentation masks produced by different methods on LiTS.}
  \label{fig:lits_viz}
\end{figure}

\begin{figure}[htbp]
  \centering
  \includegraphics[width=\textwidth]{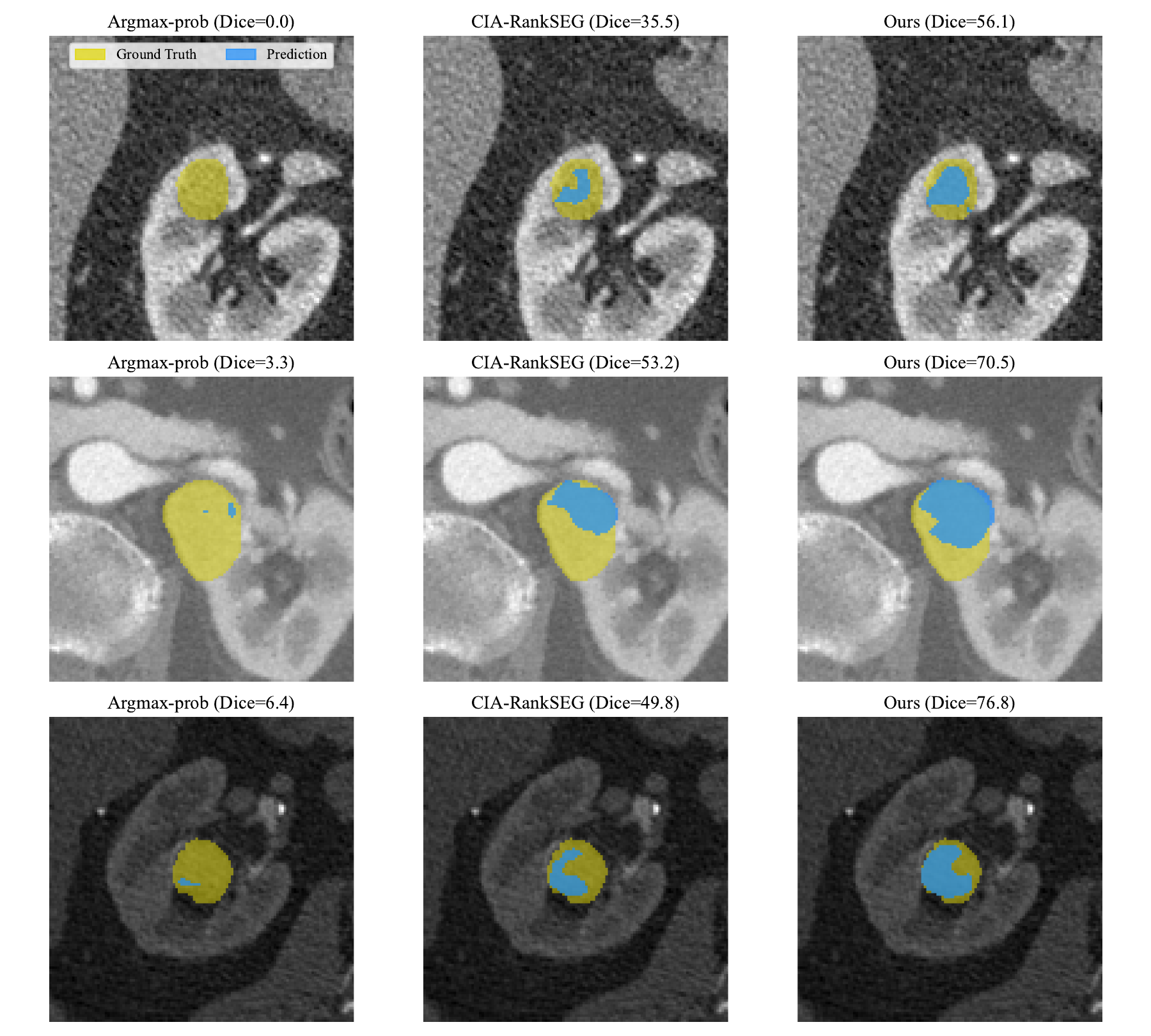}
  \caption{Qualitative comparison of segmentation masks produced by different methods on KiTS.}
  \label{fig:kits_viz}
\end{figure}

\begin{figure}[htbp]
  \centering
  \includegraphics[width=\textwidth]{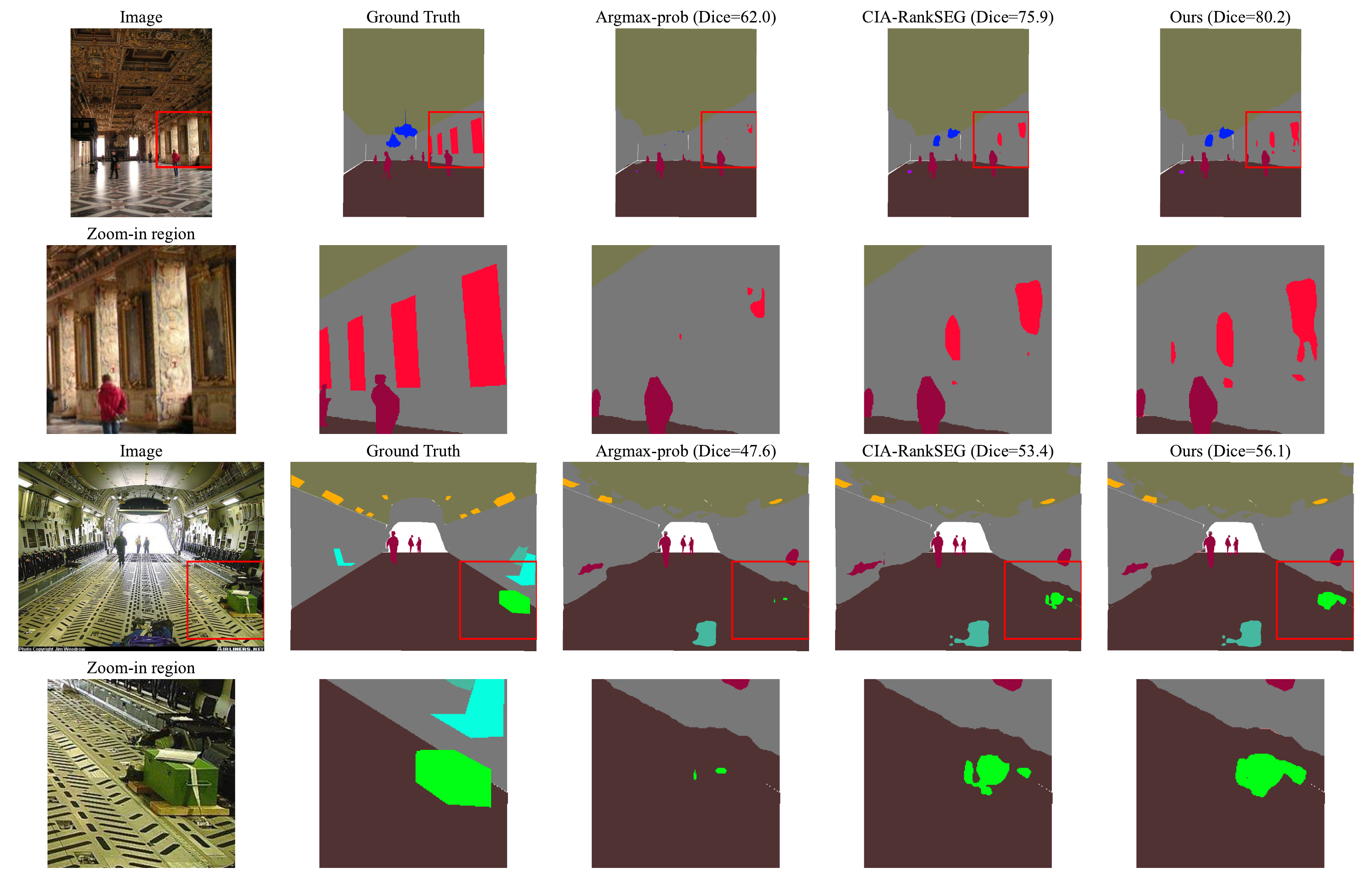}
  \caption{Qualitative comparison on ADE20K. Each case spans two rows, where the first row shows the original image and the second row zooms in on the most distinctive region.}
  \label{fig:ade20k_viz}
\end{figure}

\begin{figure}[htbp]
  \centering
  \includegraphics[width=\textwidth]{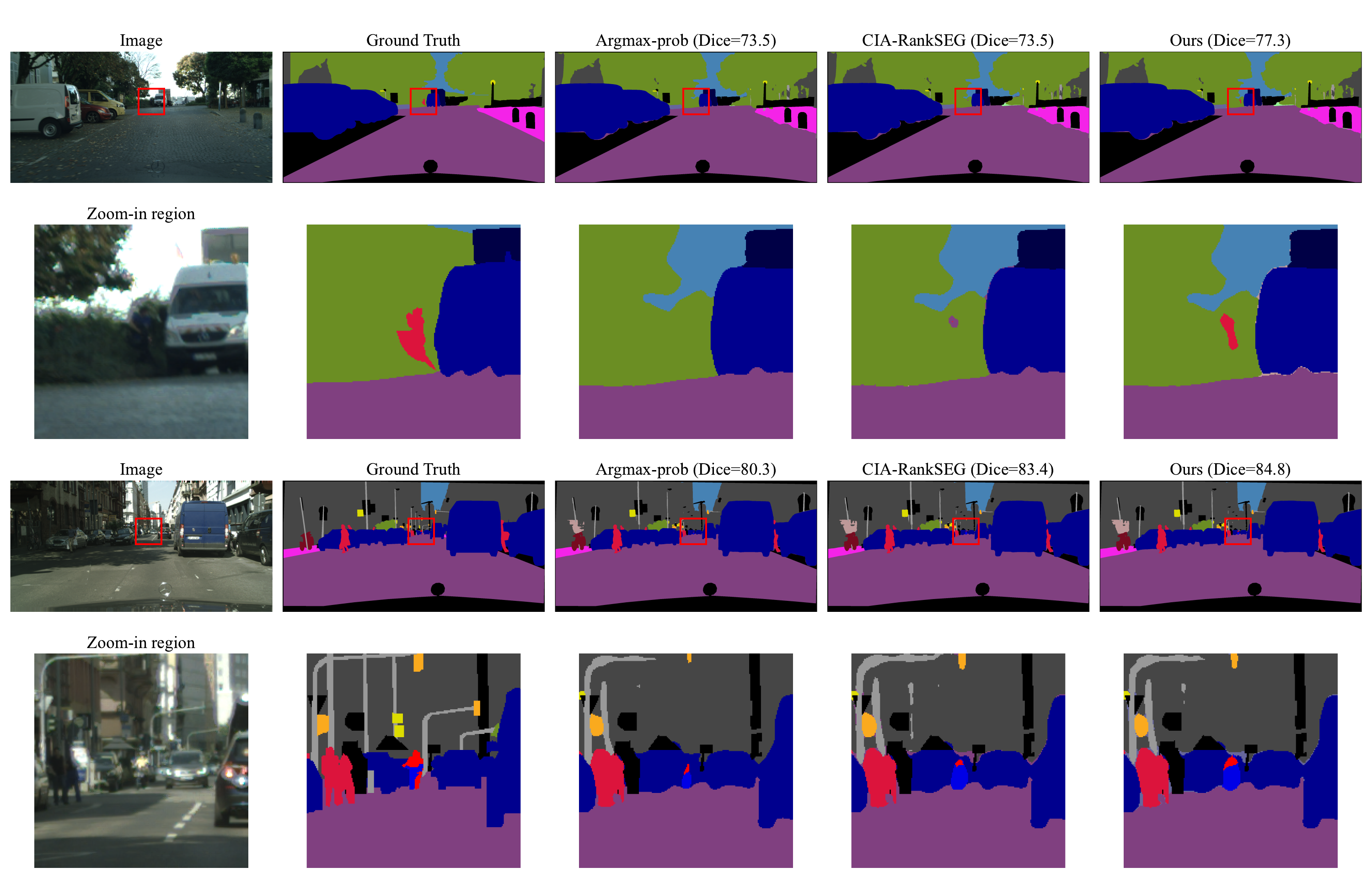}
  \caption{Qualitative comparison on Cityscapes. Each case spans two rows, where the first row shows the original image and the second row zooms in on the most distinctive region.}
  \label{fig:cityscapes_viz}
\end{figure}

\begin{figure}[htbp]
  \centering
  \includegraphics[width=\textwidth]{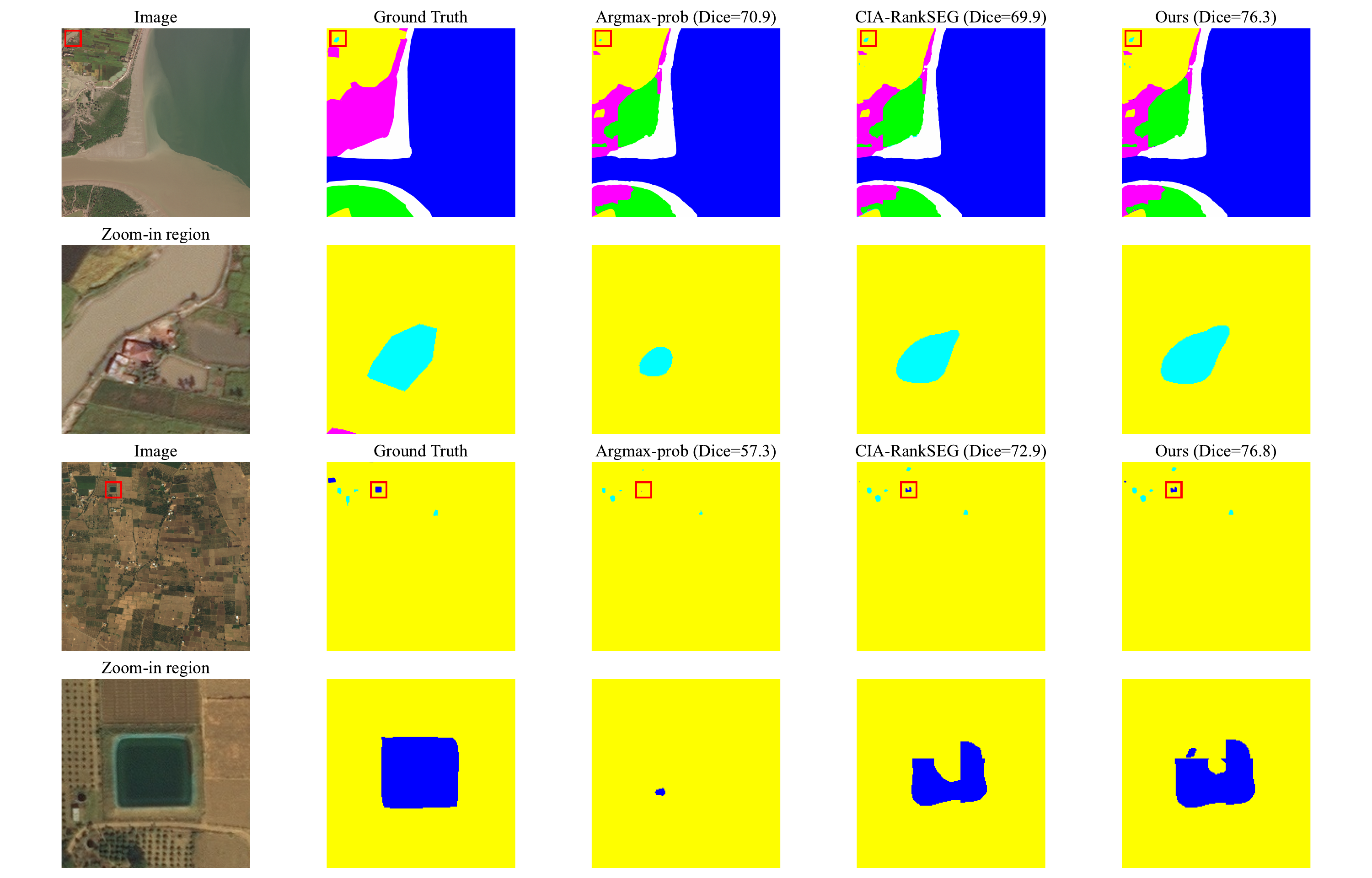}
  \caption{Qualitative comparison on DeepGlobe Land. Each case spans two rows, where the first row shows the original image and the second row zooms in on the most distinctive region.}
  \label{fig:deepglobe_viz}
\end{figure}

We complement the quantitative results in the main paper with qualitative visualizations. \cref{fig:lits_viz,fig:kits_viz} present representative examples from LiTS and KiTS, comparing the segmentation masks produced by different methods. In these medical datasets, the foreground objects (tumors) are often small and exhibit low contrast relative to the background. Such challenging conditions frequently lead to incomplete predictions by the baselines and require exploiting label dependence to recover the full object shape. In contrast, our method consistently produces more complete objects and more accurate shapes than the baselines, in line with the observed improvements in Dice and IoU scores.

\cref{fig:ade20k_viz,fig:cityscapes_viz,fig:deepglobe_viz} present examples from ADE20K, Cityscapes, and DeepGlobe Land. On these natural-image datasets, our method primarily improves the segmentation of thin and distant structures whose pixel footprint is too small for the baselines to resolve. For instance, in both examples in \cref{fig:cityscapes_viz}, our method successfully identifies or recovers more complete shapes of the distant human figures. Although these pedestrians occupy only a small number of pixels, accurately localizing them is critical for downstream tasks such as autonomous driving.


\newpage

\end{document}